\documentclass{article} 
\PassOptionsToPackage{table}{xcolor}
\usepackage{iclr2027_conference,times}

\IfFileExists{math_commands.tex}{

\usepackage{amsmath,amsfonts,bm}

\def\eqref#1{equation~\ref{#1}}

\def\1{\bm{1}}

\DeclareMathAlphabet{\mathsfit}{\encodingdefault}{\sfdefault}{m}{sl}
\SetMathAlphabet{\mathsfit}{bold}{\encodingdefault}{\sfdefault}{bx}{n}

}{}

\usepackage{hyperref}
\usepackage{url}
\usepackage{booktabs}
\usepackage{graphicx}
\usepackage{multirow}
\usepackage{amsmath,amssymb,amsthm}
\usepackage{enumitem}
\usepackage{xcolor}
\definecolor{rowhl}{gray}{0.92}
\newif\ifhastcb
\IfFileExists{tcolorbox.sty}{\hastcbtrue}{\hastcbfalse}
\ifhastcb
  \usepackage{tcolorbox}
  \tcbuselibrary{listings,skins}
  \newtcolorbox{promptbox}[1][]{
    enhanced, colback=gray!5, colframe=gray!60,
    fonttitle=\bfseries\small, title=#1,
    boxrule=0.5pt, arc=2pt, left=6pt, right=6pt, top=4pt, bottom=4pt,
    fontupper=\small\ttfamily
  }
\else
  \definecolor{pbback}{gray}{0.975}  
  \definecolor{pbframe}{gray}{0.7}   
  \newenvironment{promptbox}[1][]{%
    \par\medskip\noindent\begingroup
    \def\pbtitle{#1}%
    \setbox0\vbox\bgroup
      \hsize=\dimexpr\linewidth-13pt\relax
      \vskip4pt
      \ifx\pbtitle\empty\else{\bfseries\small\pbtitle\par\vskip2pt}\fi
      \small\ttfamily\raggedright
  }{%
    \par\vskip4pt
    \egroup
    \fboxrule=0.5pt \fboxsep=0pt
    \fcolorbox{pbframe}{pbback}{\hskip6pt\box0\hskip6pt}%
    \endgroup\par\medskip
  }
\fi
\usepackage{float}
\usepackage{wrapfig}
\makeatletter
\renewcommand\paragraph{\@startsection{paragraph}{4}{\z@}{0.5ex plus 0.2ex minus 0.1ex}{-1em}{\normalfont\normalsize\bfseries}}
\makeatother

\newtheorem{definition}{Definition}
\newtheorem{proposition}{Proposition}
\newtheorem{lemma}{Lemma}
\newtheorem{assumption}{Assumption}

\title{SynAct: Synergistic Active Feedback Learning for LLM Agents}

\author{Ning Liu, Ian Li, Kristina Klinkner\\
Amazon \\
\texttt{\{ningliun,ianliyh,klinkner\}@amazon.com}
}

\iclrpreprintcopy
\begin{document}

\maketitle

\begin{abstract}
LLM agents executing multi-step workflows fail in two fundamentally different ways: recoverable slips, which further attempts can repair, and knowledge gaps, which retrying cannot resolve. Expert feedback can address the latter, but expert attention is scarce. The key challenge is therefore not whether to ask for feedback, but \emph{where feedback can change the outcome}. Existing approaches either avoid expert feedback, allowing agents to reinforce their own errors, or allocate feedback at the task level without identifying consequential decisions. We present \textsc{SynAct}, an active learning framework that treats expert attention as the binding constraint and allocates feedback at the level of individual decisions. SynAct ranks decisions by the mutual information between the selected action and the terminal outcome across repeated rollouts, queries an expert at the highest-value decisions, and distills the resulting corrections into transferable rules that the agent can reuse without expert access at deployment. We theoretically characterize what this signal identifies under finite repeated sampling: with $k$ rollouts and at least $m$ observations per action, outcome-constant decisions have zero mutual information, while outcome-separating decisions have mutual information at least $H_2(m/k)$ bits, yielding exact separation for any threshold $t\in(0,H_2(m/k))$. This also explains why success-based acquisition can be blind when all rollouts fail: success is constant even when failed runs reach different outcomes. Across four agent benchmarks and four models spanning three families, \textsc{SynAct} improves held-out accuracy in all 16 evaluated settings, by $+11.7$pp on average. The same signal transfers to inference-time routing, where \textsc{SynAct} outperforms budget-matched cascade, uncertainty, and expert-judge routers in 15 of 16 settings, while a single targeted correction outperforms seven additional unaimed retries.
\end{abstract}

\section{Introduction}
\label{sec:intro}

An LLM agent can fail a multi-step task for two fundamentally different reasons: it can \emph{slip} despite knowing how to succeed, or it can hit a \emph{knowledge gap} that no amount of retrying will close. Additional test-time compute, through self-reflection~\citep{shinn2023reflexion,madaan2023selfrefine} or trajectory search~\citep{yao2023tot,zhou2024lats}, can often recover the first but not the second: an agent checking its own work has little basis for distinguishing a correct interpretation from a confidently wrong one~\citep{huang2024selfcorrect}. An external expert can resolve such gaps, but expert attention is scarce. Agents that sidestep the expense by distilling context from their own traces~\citep{zhang2026ace,ouyang2026reasoningbank} return judgment to the same self-assessment that fails on knowledge gaps. In real agentic systems, the missing knowledge often resides in domain-specific procedures and human expertise~\citep{liu2026eluna}, making expert attention the one resource in the loop that does not scale. This creates a fundamental allocation problem: \emph{when expert feedback is limited, where should it be spent?} A failing trajectory is not an atomic example: a 17-step episode may contain one or two decisions that determine its outcome and many that do not. Spending expert attention on the entire trajectory therefore wastes effort on irrelevant decisions. The right unit of feedback is not the trajectory, but the \emph{decision point}.

We cast expert feedback as active learning over decision points. Rather than asking which trajectories are wrong, SynAct asks a more targeted question: \emph{which action differences explain where the agent ends up?} We sample repeated rollouts of a task, group actions taken from the same environment state, and ask whether different actions lead to different terminal outcomes (Figure~\ref{fig:method}). 
We formalize this with the mutual information $I(A;O)$ between the action $A$ taken at a decision point and the run's terminal outcome class $O$. Both are observable from the agent's own rollouts. Acquisition therefore requires neither a gold answer nor a success label. This gives the signal two properties that are particularly important for feedback allocation. First, it measures \emph{consequential} disagreement rather than behavioral uncertainty: action entropy is high whenever runs disagree, whereas $I(A;O)$ is high when those action differences explain different outcomes. Second, it remains informative when every rollout fails. Different actions can still lead to different wrong end states, whereas a success-based signal collapses to zero when success is constant. SynAct can therefore surface decisions that matter precisely on tasks that are beyond the agent's current capability.

\begin{figure}[t]
\centering
\includegraphics[width=1.0\linewidth]{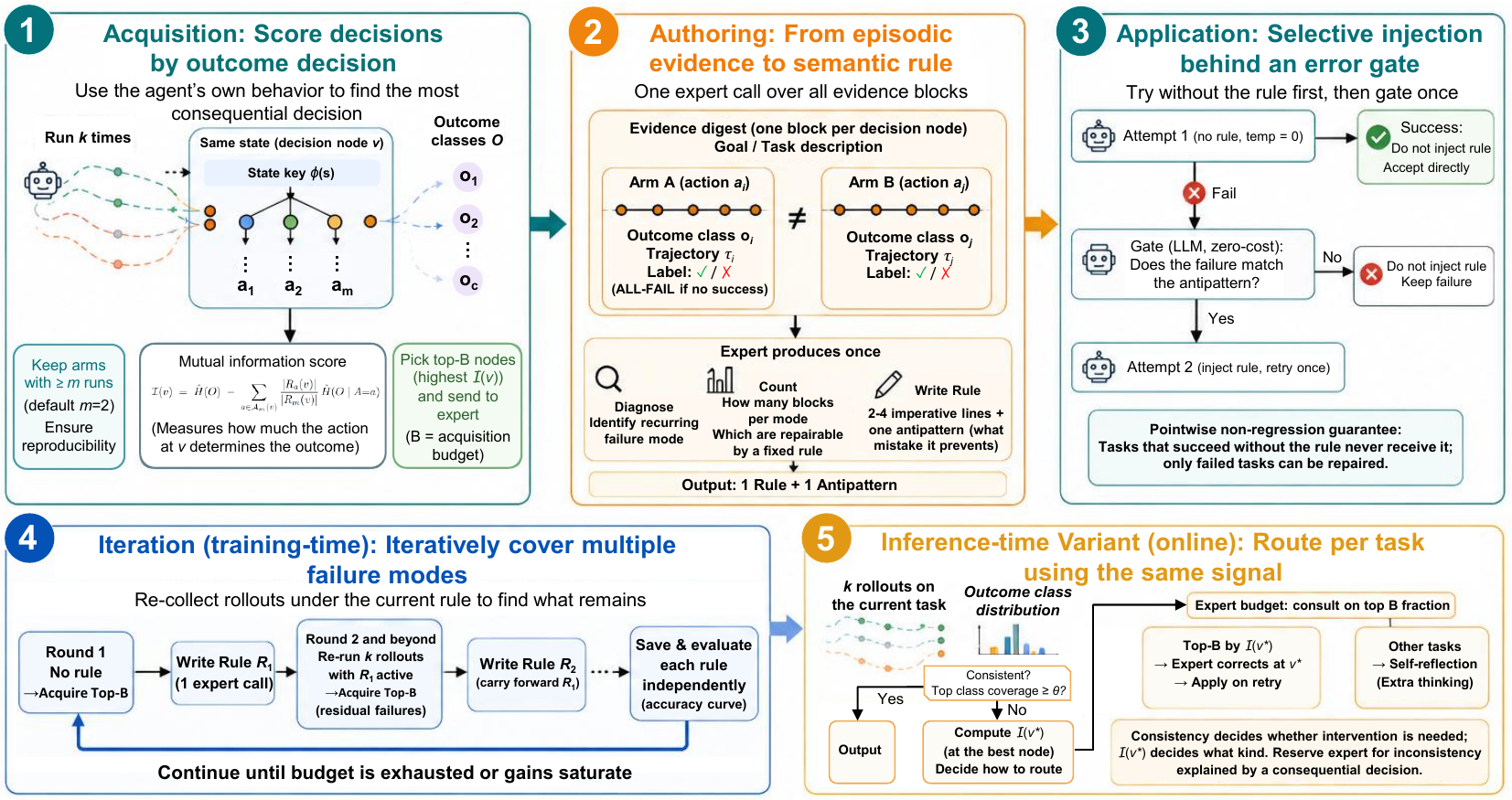}
\vspace{-5pt}
\caption{\textbf{SynAct allocates a fixed expert budget over decision points rather than episodes.} Acquisition scores each state-keyed decision node by $I(A;O)$, computed from the run itself with no gold answer or success check. Authoring sends the top-scoring nodes to the expert in one call, which abstracts them into a convention plus an antipattern. Application injects the convention only into failed attempts exhibiting that antipattern, iteration re-rolls to surface residual failures, and the same score routes consultation per task at inference time.}
\label{fig:method}
\end{figure}

This raises a second question: once a consequential decision is found, how should its correction be turned into reusable knowledge? SynAct sends the selected decision points to the expert together and asks for a compact rule: a \emph{convention} describing what to do and an \emph{antipattern} describing the failure it prevents. At deployment, an error gate applies the rule only when a failed trajectory exhibits that antipattern. This separates three decisions that are often conflated: acquisition determines \emph{where} expert attention is spent, authoring determines \emph{what} is retained, and the gate determines \emph{when} it is applied. The resulting rules outlive the trajectories that produced them, while the agent remains frozen and the expert stays out of the execution loop. We evaluate SynAct on four agent benchmarks spanning text-to-SQL, compositional crafting, shell scripting, and tool-use dialogue, using four agent models across three model families.
Our contributions are:
\begin{itemize}
\itemsep0.15em
\item We introduce \emph{budgeted active feedback acquisition over decision points} and a label-free signal, $I(A;O)$, that ranks decisions by how strongly their action choices explain downstream outcomes. It remains informative when all rollouts fail and outperforms random, entropy, and success-based selection at matched budgets.
\item We characterize the \emph{finite-sample identifiability} of this signal. With $k$ repeated rollouts and a minimum of $m$ observations per action, outcome-constant and outcome-separating decisions are exactly distinguished with any threshold $t < H_2(m/k)$, while between these extremes the score measures how much behavioral variation the outcome resolves. The same analysis proves success-based acquisition blind when all rollouts fail.
\item We make expert feedback \emph{non-regressive} through selective application: an error gate withholds a learned rule from already-solved attempts and activates it only when its antipattern is present, eliminating win-to-lose flips even when always-on context costs a weak agent up to $17.6$pp.
\item Across $4$ benchmarks $\times$ $4$ agent models with an equal expert budget, SynAct improves accuracy in all $16$ settings and leads four experience-distillation baselines in $15$. The same signal transfers to inference-time routing, where it beats budget-matched cascade, uncertainty, and expert-judge routers in $15$ of $16$ settings.
\end{itemize}

\section{Related Work}
\label{sec:related}

\paragraph{Self-improving agents.}
A large body of work builds reusable context from an agent's own traces, including playbooks~\citep{zhang2026ace}, reasoning and causal memories~\citep{ouyang2026reasoningbank,majumder2024clin,zhao2024expel}, induced workflows and tools~\citep{wang2025awm,kujanpaa2026tool}, and self-written guidelines and manuals~\citep{fu2024autoguide,chen2024automanual}. These form our primary baselines, and they differ from SynAct on two axes. Their supervision reduces to the agent's own judgment of its traces, which is least reliable where the agent is confidently wrong~\citep{huang2024selfcorrect}, and several bootstrap only from successes~\citep{wang2025awm,zhao2025sirius}, leaving no signal on tasks the agent fails throughout. They also inject context unconditionally, a regime known to introduce regressions~\citep{shao2026misevolve}. Reflective prompt search~\citep{agrawal2026gepa} is complementary, optimizing prompt space rather than allocating supervision. SynAct instead spends a small metered expert budget and gates what it learns.

\paragraph{Test-time correction and search.}
Self-generated feedback~\citep{shinn2023reflexion,madaan2023selfrefine}, search over trajectories~\citep{yao2023tot,zhou2024lats}, and compute-optimal allocation~\citep{snell2025scaling} scale the agent's own computation, which can recover slips yet cannot supply knowledge the agent lacks~\citep{huang2024selfcorrect}. Relevant to our mechanism, \citet{chen2026tta} distinguish syntactic misunderstanding of an environment's conventions from semantic misunderstanding of the task, and the former is what an authored convention repairs. Our premise is that one class of failures resists additional agent compute, and the contribution lies in where to spend a different resource.

\paragraph{Budgeted acquisition and expert intervention.}
We inherit from pool-based active learning~\citep{settles2009active}, where query-by-committee measures disagreement across hypotheses~\citep{seung1992qbc} and information-theoretic criteria select maximally informative queries~\citep{houlsby2011bald}. Our $k$ rollouts act as such a committee, but the unit of acquisition is a decision inside an episode rather than an instance, and the learner updates context rather than weights. On the intervention side, DAgger queries an expert at visited states~\citep{ross2011dagger}, preference learning extracts training signal from scarce human judgments~\citep{christiano2017preferences,liu2026beyond}, and targeted feedback recovers full-annotation quality from a fraction of labels~\citep{xu2025rlthf}, and all three economize how much supervision is needed while leaving open where within an episode it is most informative. LEAP is the closest prior work, teaching an agent from a privileged AI expert~\citep{choudhury2025leap}, though it consults the teacher per episode and fine-tunes, whereas we make one expert call per iteration and leave the model frozen. \citet{deng2026clarification} come nearest to our score, timing a clarification request by information gain, though they query the user about intent rather than an expert about a decision.

\paragraph{Credit assignment and routing.}
Localizing the step responsible for a failure has been studied through automated failure attribution, which requires a failure label our score does not~\citep{zhang2025whowhen}, and critic-free RL methods that group repeated environment states across rollouts to sharpen step-level advantages~\citep{feng2025gigpo,cheng2026graphcredit}. The grouping instinct is shared, and the consumers differ: they convert the grouped statistics into gradient updates for an unfrozen policy, whereas we convert them into a ranking that routes a metered expert to a frozen one. Process supervision localizes credit by training a step-level verifier~\citep{lightman2024verify,wang2024mathshepherd}, whereas we spend a few expert calls to author context, and the two could compose, with a verifier refining the outcome projection or breaking ties among equally ranked decisions. Uncertainty estimates over \emph{answers}~\citep{wang2023selfconsistency,kadavath2022know} and cost-aware model selection~\citep{chen2024frugalgpt,ong2025routellm,yue2024cascades} both operate on whole queries. We instead allocate expert attention within a query, at the decision whose alternatives separate the outcomes.

\section{SynAct: Synergistic Active Feedback Learning}
\label{sec:method}

The central challenge is deciding \emph{where expert attention is worth spending}. Given a fixed \emph{acquisition budget} $B$, we \textbf{acquire} the $B$ decision points that appear most consequential, have the expert \textbf{author} a reusable rule distilling what those decisions reveal, and \textbf{apply} that rule only when the current failure matches the learned error pattern. We make only one expert call per iteration. The objective is therefore to maximize the value obtained from each expert query rather than accuracy alone. We optionally repeat this process as the learned rule changes the agent's remaining failures (\S\ref{sec:iterate}).

\paragraph{Problem setup.}
An \emph{agent} is a small or medium open-weights model $\pi$ that interacts with an environment through a ReAct-style loop~\citep{yao2023react}. An \emph{expert} provides knowledge the agent does not have: a subject-matter expert in deployment, instantiated in our experiments by a strong model (Claude Opus~\citep{claudeopus}). The expert is available during training only, leaving the agent to operate alone at test time. A rollout is a sequence of state--action pairs $\tau=(s_0,a_0,\dots,s_{T})$, with $\mathrm{succ}(\tau)\in\{0,1\}$ denoting its outcome under the benchmark checker.

\subsection{Acquisition: Finding Consequential Decisions}
\label{sec:acq}

The ideal acquisition score would tell us how much an expert correction at each decision would improve the agent. That quantity is counterfactual: estimating it requires intervening on the decision and observing what would have happened, often before we know the task's answer. We therefore extract a proxy for consequence from the agent's own behavior. We run the same task multiple times, find decisions where different runs make different choices \emph{in the same state}, and ask whether those choices lead to different terminal outcomes. If all alternatives converge to the same outcome, the decision provides little evidence that changing it would matter. If different alternatives lead to different outcomes, the decision becomes a natural target for intervention. We capture this contrast with the mutual information $I(A;O)$ between the action at a decision and the rollout's terminal outcome class. Ranking by this signal concentrates useful decisions near the top, with the leading decile capturing most of the benefit (\S\ref{sec:res-budget}).

\paragraph{State-keyed grouping.}
For each training task, we independently sample $k$ rollouts from $\pi$ at temperature $>0$ ($k{=}8$ by default; Appendix~\ref{app:setup-hparam}) and group steps across rollouts by the environment state. Let $\phi(s)$ denote a state key: a normalized hash of the observation at $s$, canonicalized per benchmark by lower-casing, collapsing whitespace, and removing volatile fields such as step counters. A \emph{decision node} $v$ is a maximal set of (run, step) pairs sharing a state key, with $R(v)$ the runs reaching $v$, partitioned into \emph{arms} by the next action. We score only arms supported by at least $m{=}2$ runs, since an arm observed once cannot provide reliable within-state evidence. We key on state rather than trajectory position because two runs may reach the same state through different action sequences while facing exactly the same decision, a correspondence position-based grouping discards. We additionally verify the key against an order-invariant alternative: the acquired set changes on $5\%$ of tasks, and an expert judge deems the additional merges largely illegitimate (Appendix~\ref{app:grouping}).

\begin{assumption}[Decision identifiability]
\label{as:ident}
The state key $\phi$ is \emph{decision-sufficient}: if $\phi(s)=\phi(s')$ then the two states admit the same action set and induce the same conditional outcome distribution for each action. The outcome projection $\omega$ depends only on the terminal state. Conditioned on that state, $O$ is therefore independent of $A_v$.
\end{assumption}

Assumption~\ref{as:ident} makes the comparison meaningful. The first clause ensures that states sharing a key represent the same decision, while the second ensures that the outcome projection measures what happens \emph{after} the decision rather than encoding the decision itself. Without the latter, an $\omega$ that reads the branching action would make every decision consequential by construction. Deterministic environments sharpen the conditional outcome distributions but guarantee neither clause: both depend on the sufficiency of $\phi$ and $\omega$ (Appendix~\ref{app:scope}).

\paragraph{Outcome projection.}
The outcome variable should capture a decision's consequence without requiring a gold answer. Let $A_v$ be the action taken at $v$ and $O=\omega(\tau)$ the rollout's \emph{terminal outcome class}, computed from its own trajectory. We assign the same class to runs that end in the same state. In practice, we derive $O$ from the normalized final \texttt{SELECT} on SQL, sorted \texttt{craft}/\texttt{get} targets on TextCraft, the final non-navigation command on Bash, and the write-tool multiset on $\tau$-bench. Coarsening or refining these projections shifts the ranking but leaves the authored rule unchanged or improves it (Appendix~\ref{app:outcome}).

\paragraph{Acquisition score.}
We now measure how strongly the terminal outcome distinguishes the alternatives at a decision. Let $\mathcal{A}_m(v)=\{a : |R_a(v)| \ge m\}$ denote the supported arms and $R_m(v)=\bigcup_{a \in \mathcal{A}_m(v)} R_a(v)$ the runs they contain. Estimating both entropies by plug-in counts over $R_m(v)$ gives
\begin{equation}
  \mathcal{I}(v) \;=\; \hat{H}(O) \;-\; \sum_{a \in \mathcal{A}_m(v)} \frac{|R_a(v)|}{|R_m(v)|}\, \hat{H}(O \mid A{=}a),
  \label{eq:mi}
\end{equation}
the standard decomposition $I(A;O)=H(O)-H(O\mid A)$ evaluated on the supported population. The restriction to supported arms keeps the arm weights normalized and ensures $\mathcal{I}(v)\ge 0$. We retain the highest-scoring node from each task, rank tasks by that score, and send the top $B$ to the expert. Acquisition therefore requires neither a gold label nor a success check: the agent's own repeated behavior determines where expert attention is most likely to matter.

\subsection{Identifiability Under Repeated Sampling}
\label{sec:theory}

The acquisition score comes from only $k$ rollouts. The natural question is what it can reliably identify at finite sample sizes. We begin with two extremes the score should distinguish: decisions whose alternatives converge to one outcome and decisions whose alternatives lead to disjoint outcomes. Once a node is scorable, this distinction is exact for every realization.

\begin{definition}[Outcome-constant and outcome-separating nodes]
\label{def:nodes}
A scorable node $v$ is \emph{outcome-constant} if $\omega$ is constant on $R_m(v)$, so its runs reach one outcome class whichever arm they take. It is \emph{outcome-separating} if the arms of $\mathcal{A}_m(v)$ have disjoint outcome supports, that is $\omega(R_a(v))\cap\omega(R_{a'}(v))=\emptyset$ for $a\ne a'$, so the outcome determines which arm was taken.
\end{definition}

\begin{proposition}[Exact separation at finite $k$]
\label{prop:sep}
Fix $m\ge 2$ and let $v$ be scorable, meaning $|\mathcal{A}_m(v)|\ge 2$. If $v$ is outcome-constant then $\mathcal{I}(v)=0$ for each realization of the rollouts, and if $v$ is outcome-separating then $\mathcal{I}(v)=\hat{H}(A_v)\ge H_2(m/k)$, with $H_2$ the binary entropy in bits. Hence any threshold $t\in\bigl(0,H_2(m/k)\bigr)$ classifies both types with zero error for every realization, and at the default $k{=}8$, $m{=}2$ the admissible range is $t\in(0,0.811)$.
\end{proposition}

We prove both clauses in Appendix~\ref{app:proof-prop}, and we can read the result directly off Equation~\ref{eq:mi}: an outcome-constant node has $\hat{H}(O)=0$, an outcome-separating node has $\hat{H}(A_v\mid O)=0$ so its mutual information reaches $\hat{H}(A_v)$, and the support floor supplies the finite-sample gap between them, since any arm appearing in at least $m$ of the $k$ rollouts forces $\hat{H}(A_v)\ge H_2(m/k)$. Sampling noise therefore cannot collapse the two extremes. What sampling does control is \emph{scorability}: since $N_{v,a}\sim\mathrm{Binomial}(k,p_{v,a})$, we retain an arm taken half the time with probability $0.965$ at $k{=}8$, but one taken a tenth of the time only with probability $0.187$ (Proposition~\ref{prop:recovery}). The support floor thus trades coverage for reproducibility, keeping the alternatives a written rule can address.

\paragraph{Partial identifiability.}
Most decisions lie between the two extremes. In this regime, $\mathcal{I}(v)=\hat{H}(A_v)-\hat{H}(A_v\mid O)$ (Appendix~\ref{app:proof-lemma}): the score measures the fraction of behavioral variability at $v$ that the outcome resolves, identifying a decision to the extent that its alternatives produce distinguishable outcomes (Proposition~\ref{prop:partial}, Appendix~\ref{app:proof-partial}). The plug-in estimator also carries a positive null bias, $\mathbb{E}[\mathcal{I}(v)]>0$ even for outcomes independent of the action. We use $\mathcal{I}(v)$ to \emph{rank} decisions rather than as a hypothesis test. We compute the estimator's exact null distribution, its growth with the number of outcome classes, and its ranking stability at the default budget in Appendix~\ref{app:proof-null}.

\paragraph{Why outcomes rather than actions or success?}
Action entropy $\hat{H}(A_v)$ detects behavioral variation but cannot tell whether it matters: an outcome-constant and an outcome-separating node can share the same action entropy, while $\mathcal{I}(v)$ is $0$ for the former and $\hat{H}(A_v)$ for the latter. Success creates the opposite problem: replacing $O$ with $\mathrm{succ}(\tau)$ makes an all-fail task outcome-constant, hiding exactly the tasks for which the agent most needs external knowledge (Proposition~\ref{prop:allfail}, Appendix~\ref{app:proof-allfail}). Our outcome projection instead preserves distinctions among \emph{wrong} end states, keeping the signal informative even when no rollout succeeds. Empirically, $I(A;O)$ routes $1.9\times$ more decisions than a success-based signal ($274$ against $146$ on one SQL model), with $20$ of $49$ acquired Qwen TextCraft tasks coming from all-fail cases (Appendix~\ref{app:coverage}).

\subsection{Authoring: From Episodic Evidence to Semantic Rule}
\label{sec:author}

Acquisition identifies \emph{where} expert attention is valuable, and authoring determines \emph{what} reusable knowledge that evidence supports. For each acquired node, we construct an evidence block containing the task goal and two arms with different terminal outcomes, each with its diverging action, the trajectory that followed, and a \texttt{correct}/\texttt{wrong} tag from its success flag. When no rollout succeeds, we instead mark the block \texttt{ALL-FAIL} and provide two distinct wrong end states. Acquisition thus remains label-free, while authoring may use success annotations where they exist.

A single decision can expose a mistake without revealing whether that mistake recurs across tasks. We therefore aggregate all acquired blocks and give the expert the full digest in \textbf{one call}. We ask the expert to \emph{diagnose} recurring failure modes, \emph{count} the blocks each mode explains, mark those a fixed convention could repair, and \emph{write} one convention of 2--4 imperative lines covering the dominant repairable modes, with individual bullets free to address distinct modes, plus one \emph{antipattern} summarizing the family of mistakes the convention prevents. This aggregation turns isolated episodes into a transferable rule, and the single call decouples expert cost from the acquisition budget: more evidence costs no additional expert invocations.

\subsection{Application: Selective Injection Behind an Error Gate}
\label{sec:gate}

A useful convention can still hurt when applied to the wrong failure. SynAct lets the current attempt determine whether the learned rule should be used on the next one. Each task receives at most two attempts. We first run the agent with no convention at temperature $0$ and stop if it succeeds. Otherwise, an LLM gate inspects the failing trajectory for the learned antipattern. Only when the gate fires do we prepend the convention and retry once. The gate sees neither the success bit nor the reference solution. We use the agent model for gating rather than the expert, so this step adds no expert cost (Appendix~\ref{app:setup-gate}).

This design gives a pointwise non-regression guarantee: since a base success never sees the convention, the rule cannot convert it into a failure, and any improvement must come from repairing failures. We report accuracy as base successes plus gate-repaired failures and track $n_{\text{fired}}$, the tasks reaching the gated retry. In Appendix~\ref{app:gate-detail} we evaluate the gate against unconditional injection, validate its firing decisions against an expert judge, redeploy the baselines' knowledge behind the same gate, and substitute a weaker authoring expert.

\subsection{Iteration and Inference-Time Reuse}
\label{sec:iterate}

\paragraph{Iterative acquisition.}
Learning a rule changes the failure distribution. Once the first convention repairs some errors, previously hidden failure modes become more prominent, and the original rollouts cannot reveal this shift because they were generated before the rule existed. SynAct therefore re-collects after each update. At iteration~1, we acquire from base rollouts. At iteration $i>1$, we generate $k$ fresh rollouts with the current convention active, score the residual failures, and ask the expert to \emph{revise} the convention: retaining bullets supported by the new evidence, adding bullets for newly surfaced modes, and removing bullets the evidence contradicts. We rewrite the antipattern alongside the convention so that the gate always evaluates the current rule, and each iteration thus spends its single expert call on what remains unrepaired. We save and evaluate every iteration's rule independently, turning iteration into a budget--performance curve rather than a fixed schedule (\S\ref{sec:res-budget}).

\paragraph{Inference-time variant.}
\label{sec:online}
The same signal can allocate expert attention directly when the expert remains available at deployment, as in an LLM cascade~\citep{chen2024frugalgpt,yue2024cascades}. Because $\mathcal{I}(v)$ is computed from $k$ rollouts of the task at hand, the separation result applies under the same conditions with no training set. Instead of amortizing the feedback into a rule, SynAct spends the signal immediately. We first draw $k$ rollouts, and if the task is \emph{behaviorally consistent}, meaning its modal outcome class covers at least a fraction $\theta$ of runs, we return the modal answer at no further cost. Otherwise, under a budget permitting expert consultation on a fraction $B$ of tasks, we send the top $B$ inconsistent tasks by $\mathcal{I}(v^\star)$ to the expert, which supplies a correction at $v^\star$ for the agent to apply on retry, and the remainder receive additional self-reflection~\citep{shinn2023reflexion} at a matched attempt budget, equivalently a per-benchmark threshold on $\mathcal{I}(v^\star)$.

Consistency thus determines \emph{whether} to intervene, while $\mathcal{I}(v^\star)$ determines \emph{where} expert attention is worth spending: SynAct reserves the expert for cases where behavioral disagreement points to a consequential decision, and resolves the remainder with cheaper self-reflection. Offline, the expert teaches once and the resulting rule deploys without the expert, while online the expert is consulted per difficult task and needs no training set (\S\ref{sec:res-online}).

\section{Experimental Setup}
\label{sec:setup}

\paragraph{Benchmarks.}
We evaluate on four agent benchmarks with binary success checks: InterCode-SQL on Spider ($n_{\text{test}}{=}150$~\citep{yang2023intercode,yu2018spider}), TextCraft at recipe depth 3 ($n{=}58$~\citep{prasad2024adapt}), InterCode-Bash on nl2bash ($n{=}40$~\citep{yang2023intercode,lin2018nl2bash}), and $\tau$-bench retail ($n{=}115$~\citep{yao2024taubench}), training on a disjoint split of up to 50 tasks each. We select these four because each exposes an observable state that recurs across rollouts, which is what makes a state key well defined under Assumption~\ref{as:ident}, and because they span text-to-SQL, compositional crafting, shell scripting, and tool-use dialogue (Appendix~\ref{app:setup-bench}).

\paragraph{Agent models and expert.}
We run four models from three families, Qwen3-32B~\citep{qwen3}, Gemma-3-12B and Gemma-3-4B~\citep{gemma3}, and Ministral-3-3B~\citep{ministral3}, and substitute Claude Opus~\citep{claudeopus} for the subject-matter expert, held fixed across all $16$ settings and iterations. The substitution is conservative: the expert sees only the acquired decision points and returns a short convention, an interface a human could equally serve, though whether one would author the same convention is untested. We span weak-to-mid agents, where expert feedback has the most headroom, and Qwen3-32B is the strong control (Appendix~\ref{app:setup-models}).

\paragraph{Baselines.}
We compare against four experience-distillation baselines and give each the same expert SynAct receives. The comparison thus isolates \emph{where} expert attention is pointed from \emph{whether} it is available: \textbf{ACE}~\citep{zhang2026ace} grows an evolving playbook of tactics prepended to the prompt, \textbf{ReasoningBank}~\citep{ouyang2026reasoningbank} retrieves distilled strategies from a memory of past episodes, \textbf{AWM}~\citep{wang2025awm} induces workflow snippets from successful trajectories, and \textbf{AWM-iter} runs AWM for ten iterative expert induction rounds, matching our iteration count. We tune setup parameters alone and leave the methods unmodified (Appendix~\ref{app:setup-baselines}). In \S\ref{sec:res-signal} we add \emph{acquisition-function} baselines holding our pipeline fixed and varying only the selection rule: SynAct\nobreakdash-Succ ($I(A;\mathrm{succ})$), SynAct\nobreakdash-Ent (action entropy), and SynAct\nobreakdash-Rand (uniform).

\paragraph{Protocol.}
We report each setting as the mean of 3 repeats on the full test set, holding the rule, gate, and test set fixed except for the ablated component. The gate re-measures the unaided agent at each iteration, letting us estimate measurement noise directly: base accuracy varies by $3.4$--$8.5$pp across iterations. We pool base where possible and treat differences of this size between adjacent iterations as indistinguishable (Appendix~\ref{app:setup-protocol}). Agent and expert tokens are reported separately because expert tokens cost roughly $10$--$30\times$ more; turns are exact and tokens estimated (Appendix~\ref{app:cost}).

\section{Results}
\label{sec:results}

\subsection{Routed Expert Feedback Against Experience Distillation}
\label{sec:res-main}

\paragraph{Accuracy against experience distillation.}
Distilling experience into context is the prevailing way to improve a frozen agent. We first ask what routing the same expert to decision points buys over distilling with it. We give ACE, ReasoningBank, AWM, and ten-round AWM-iter the identical expert and hold the training split, test set, and agent fixed (Figure~\ref{fig:train-headline}, Table~\ref{tab:main}). SynAct improves on base in all 16 settings, by $+8.3$pp on average after a single iteration and $+11.7$pp at the best iteration. Taking the strongest baseline per setting, an oracle over the four, yields $+4.4$pp on average, while no fixed baseline exceeds $+1.0$pp. SynAct at one iteration therefore nearly doubles the gain of even the per-setting oracle. Exact sign tests over the 16 settings confirm the advantage against base, each baseline, and the oracle at $p \le 2.1\times 10^{-3}$ (Appendix~\ref{app:tables}). At its best iteration, SynAct leads all four baselines in 15 of 16 settings, the exception being Ministral TextCraft, where AWM reaches $20.7$ against our $19.5$. The largest gains occur where the base agent is weakest, $+24.8$pp for Gemma-4B on Bash and $+18.0$pp for Ministral on SQL, which identifies SynAct's most favorable regime: agents that can reach informative end states yet lack the knowledge to reach the correct one. The controlled comparison points to \emph{where} attention is spent: distillation summarizes whole trajectories, whereas SynAct routes the expert to the decisions determining outcomes (Appendix~\ref{app:selfjudge}).

\begin{figure}[t]
\centering
\includegraphics[width=1.0\linewidth]{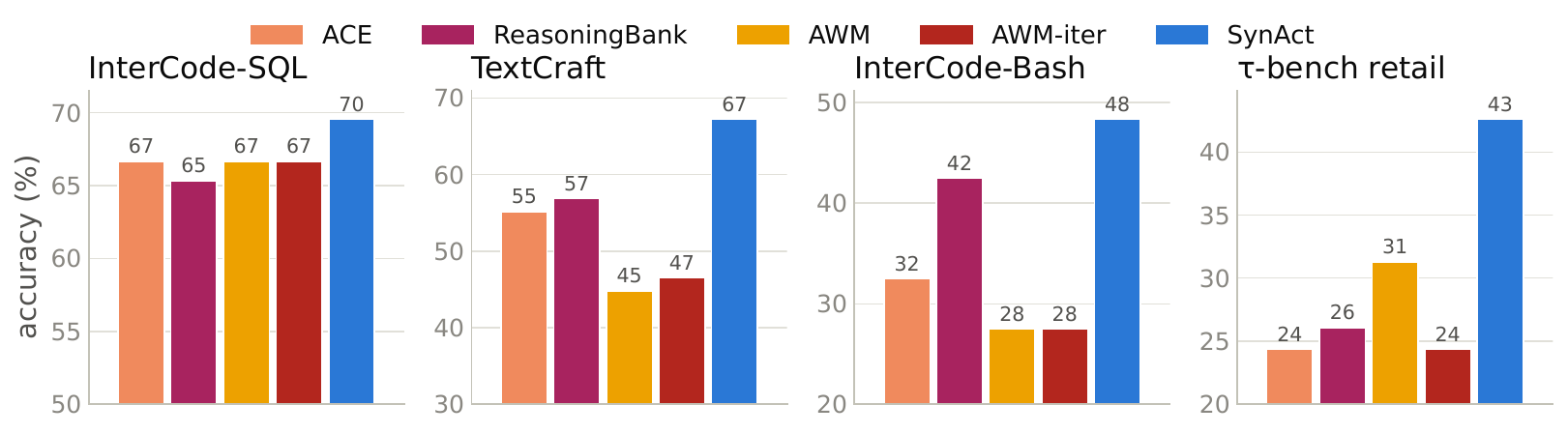}
\vspace{-10pt}
\caption{\textbf{Training-time methods with an equal expert budget} (agent Qwen3-32B, one panel per benchmark): ACE, ReasoningBank, AWM, AWM-iter, and SynAct at its best iteration, all given the same expert. Table~\ref{tab:main} reports all four agent models.}
\label{fig:train-headline}
\end{figure}

\paragraph{Downside risk.}
Mean accuracy conceals whether a method repairs failures or exchanges them for new errors. To separate the two, we compare each method's worst setting to its base. SynAct stays at or above base within noise in every setting, whereas always-on injection can substantially reduce accuracy: ReasoningBank costs Gemma-4B $17.6$pp on SQL ($47.6\rightarrow30.0$), and ACE costs Ministral $3.6$pp on the same benchmark. The damage concentrates on the weakest agents, the population with the most headroom, and the distinction is selective application. Holding the rule fixed and varying only its application, unconditional injection yields 32 to 36 win$\rightarrow$lose flips on SQL, compared with $0$ when gated. We therefore deploy the gate without an additional validation step, and we further validate its firing decisions against an expert judge, which the gate rarely misses (Appendix~\ref{app:gate-detail}).

\subsection{Comparison of Acquisition Functions}
\label{sec:res-signal}

Consulting a strong model helps on its own. We therefore ask whether the choice of decisions drives the gain. We hold the rollouts, authoring call, and consistency gate fixed, varying only the selection rule. Outcome-MI is the best rule in all 8 swept settings and all 23 pairwise comparisons, gaining $+12.5$pp over base, compared with $+8.2$ for random selection, $+6.9$ for success-MI, and $+6.8$ for entropy. The entropy comparison reveals why outcome information matters: when two actions differ frequently while converging on the same outcome, ranking by disagreement alone spends the budget on variation without consequence, precisely the distinction Proposition~\ref{prop:sep} captures. Success-MI fails in the complementary regime, remaining silent on all-fail tasks where the success label is constant (Appendix~\ref{app:coverage}). Under the final protocol, the four rules yield indistinguishable one-iteration gains, suggesting that outcome information matters where the theory predicts: all-fail coverage and per-task routing, while the training loop is largely insensitive to the ranking (Appendices~\ref{app:budget} and~\ref{app:signal-routing}).

\subsection{Expert-Token and Sampling Budgets}
\label{sec:res-budget}

\paragraph{Cost of a routed expert budget.}
\label{sec:res-cost}
\begin{wrapfigure}{r}{0.39\linewidth}
\centering
\vspace{-39pt}
\includegraphics[width=\linewidth]{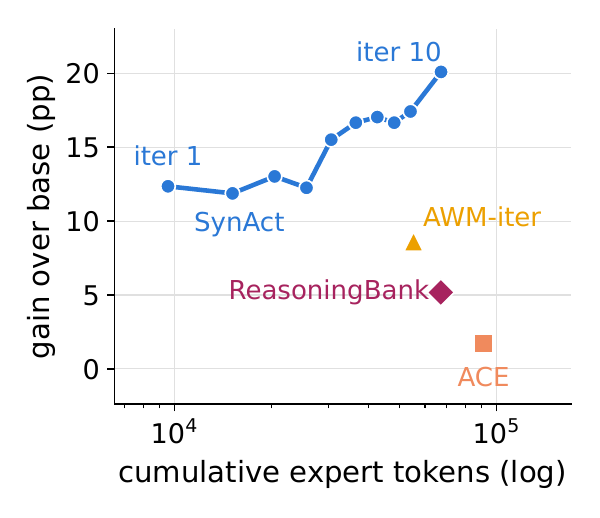}
\vspace{-25pt}
\caption{\textbf{Gain vs. cumulative expert tokens} (TextCraft, Qwen3-32B, log $x$): SynAct over ten iterations, distillers at their measured budgets.}
\label{fig:scaling}
\vspace{-20pt}
\end{wrapfigure}
With the expert held fixed, we next measure how much accuracy each method buys per expert token. Figure~\ref{fig:scaling} plots SynAct's gain over base across ten iterations against the cumulative expert tokens spent authoring the rule, with the one-shot distillers at their own measured budgets. The distillers cluster near zero gain at comparable spend, whereas SynAct's gain grows through the first several iterations before flattening: routed expert tokens keep buying accuracy after the distillers have saturated. Expert turns diverge further, one call per iteration against one or more per task (Appendix~\ref{app:cost-detail}).

\paragraph{Sensitivity to sampling budget and iteration count.}
Since rollouts dominate the cost, we turn both knobs down. All six swept settings beat base at $k{=}2$, a $4\times$ cheaper collection, and the top decile of MI-ranked decisions already delivers most of the full-budget gain (Table~\ref{tab:budget}, Appendices~\ref{app:budget} and~\ref{app:inflight}). Iteration then compounds what a single pass leaves: the first iteration captures $70\%$ of the median setting's best-iteration gain, and later passes lift Qwen TextCraft from $+11.4$pp to $+20.6$pp. Across all $160$ iteration evaluations, exactly one dips below base, by $0.2$pp within noise (Table~\ref{tab:curve}, Appendix~\ref{app:iter}).

\subsection{Inference-Time Routing}
\label{sec:res-online}

\begin{figure}[t]
\centering
\includegraphics[width=\linewidth]{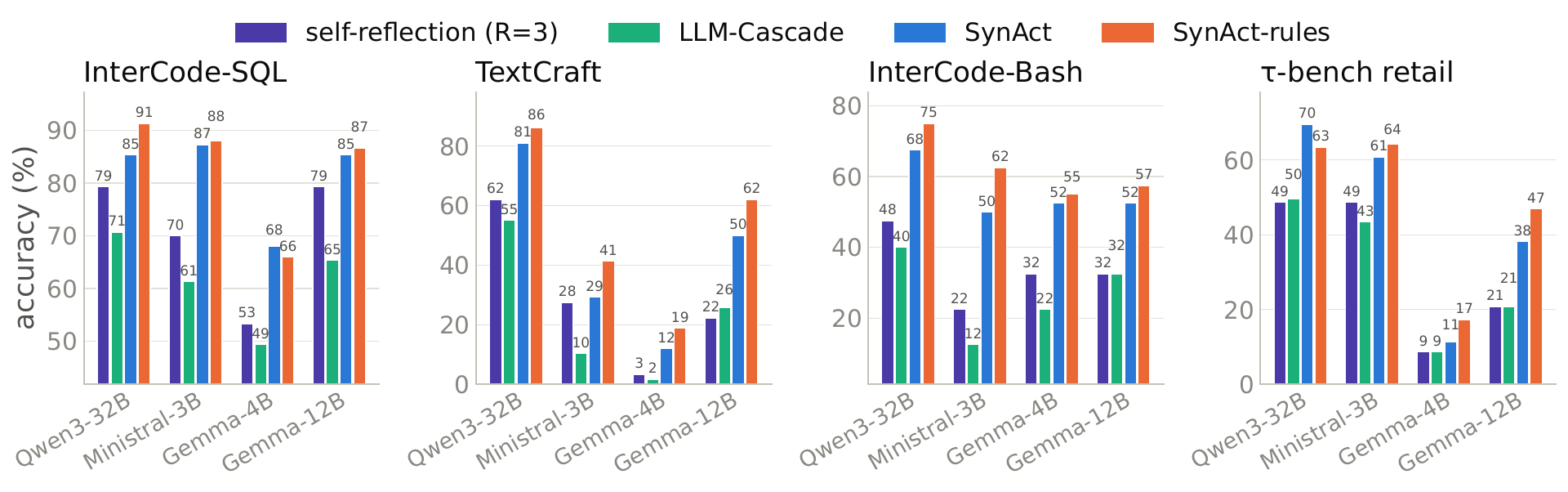}
\vspace{-10pt}
\caption{\textbf{Inference-time routing at a matched expert budget} ($\theta{=}0.6$, $R{=}3$, one panel per benchmark): self-reflection with ten retries, a token-matched cascade, SynAct, and SynAct with the iteration-10 rule active. Table~\ref{tab:online} adds the action-entropy and expert-judge routers.}
\label{fig:online-headline}
\end{figure}

We finally test whether the allocation principle transfers to inference, matching each router to SynAct's budget (\S\ref{sec:iterate}) and varying only its aim: a token-matched cascade~\citep{chen2024frugalgpt,yue2024cascades}, an action-entropy router, and an expert judge routing tasks itself (Figure~\ref{fig:online-headline}, Table~\ref{tab:online}). SynAct or its rule-augmented variant is the best router in 15 of 16 settings, and on tasks both routers admit, the outcome-MI branch rescues up to $2.7\times$ more (Appendix~\ref{app:signal-routing}). Against self-reflection with ten retries~\citep{shinn2023reflexion}, three attempts plus routed consultation win 14 of 16 settings and add $+11.6$pp on average: one aimed correction outperforms seven unaimed retries. The token-matched cascade never exceeds the retry baseline, consistent with whole-answer escalation wasting budget on localized failures.

The learned rule and online routing also compose: activating the iteration-10 convention alongside routing lifts accuracy in 14 of 16 settings over routing alone, by $+5.1$pp on average and $+12.1$pp for Gemma-12B TextCraft, as the standing rule repairs the recurring mode and frees the routed call. Raising the retry budget from $R{=}3$ to $10$ never hurts and helps 11 of 16 settings (Appendix~\ref{app:online}).

\section{Conclusion}
\label{sec:conclusion}

We recast improving an LLM agent as active learning over decision points: under scarce expert attention, the question is not which trajectories to review, but which decisions are worth explaining. SynAct answers with the mutual information between an action at a state and the terminal outcome class, which under decision-sufficient projections exactly separates outcome-separating from outcome-constant decisions at finite $k$ while remaining informative where success-based scores are blind. One expert call per iteration turns this signal into a gated convention improving all 16 settings, while four distillers using the same expert budget average under $+1$pp. The same signal also routes consultation at deployment, suggesting a natural next step: closing the loop with human experts, where turns rather than tokens are the scarce resource.

\paragraph{Limitations.}
Because Assumption~\ref{as:ident} constrains the projections rather than the dynamics, SynAct applies to stochastic environments, though we do not measure that degradation. The label-free claim covers acquisition alone, and the consistency gate can miss confidently wrong but behaviorally consistent failures. Our expert is a proxy for a strong human, leaving the cost of human authoring untested, and the support floor can miss rare modes (Appendices~\ref{app:gate-detail}, \ref{app:proof-recovery}).

\subsection*{AI use statement}
In this work, we used generative AI tools for designing or providing feedback on research methodology or experiments, implementing methods, and interpreting results. We have not used generative AI tools for generating synthetic data sets, developing theoretical models or conceptual frameworks, formulating mathematical claims, providing critical ingredients for proving mathematical claims, assisting in the writing of proofs, or proposing or refining hypotheses. Assisting with translation, cleaning and reformatting datasets, and supporting qualitative and thematic data analysis are not applicable to this work. Additionally, we used generative AI tools for editing the paper to improve readability and for identifying relevant literature. We have reviewed all AI-assisted work. LLM-generated code was verified and tested for correctness, AI-suggested interpretations were checked against outputs, and AI-identified references were verified against their source records. We take responsibility for the final content of this work, including text, claims or artifacts produced with the aid of generative AI.

\subsection*{Reproducibility statement}
The method is specified completely in \S\ref{sec:method}, covering the state key and grouping (\S\ref{sec:acq}, Appendix~\ref{app:grouping}), the per-benchmark outcome projections $\omega$ (Appendix~\ref{app:outcome}), the acquisition score with its $m{\ge}2$ arm floor (Eq.~\ref{eq:mi}), the verbatim authoring prompt and output schema (Appendix~\ref{app:example}), and the two-attempt gated application protocol (\S\ref{sec:gate}). Appendix~\ref{app:setup} specifies splits, step limits, models, hyperparameters, the gate, baseline configurations, and the measurement protocol. All benchmarks are public~\citep{yang2023intercode,prasad2024adapt,yao2024taubench} and agent models are served through a public API. Each number in the iteration and budget tables comes from a script that refuses to print any entry lacking all 3 repeats: partially-evaluated results keep moving, and no number is transcribed by hand. The harness also distinguishes a task the agent failed from a call that never completed, retrying transport and throttling errors rather than recording them as failures.

\bibliography{nliu}
\bibliographystyle{iclr2027_conference}

\appendix

\section{Identifiability of Consequential Decisions}
\label{app:proofs}

We characterize what $\mathcal{I}(v)$ can identify from only $k$ rollouts, proceeding from its empirical mutual-information representation to finite-sample separation, partial identifiability, sampling recovery, ranking behavior, and the assumptions these guarantees require. For a decision node $v$, let $\mathcal{A}_m(v)=\{a:|R_a(v)|\ge m\}$ be the arms clearing the support floor, $R_m(v)=\bigcup_{a\in\mathcal{A}_m(v)}R_a(v)$, $N=|R_m(v)|$, and $n_a=|R_a(v)|$. We compute entropies as plug-in estimates in bits on $R_m(v)$, call $v$ \emph{scorable} when $|\mathcal{A}_m(v)|\ge 2$, and write $H_2(x)=-x\log_2 x-(1-x)\log_2(1-x)$. Restating Eq.~\ref{eq:mi},
\begin{equation}
\mathcal{I}(v)=\hat{H}(O)-\sum_{a\in\mathcal{A}_m(v)}\frac{n_a}{N}\hat{H}(O\mid A{=}a).
\label{eq:mi-restate}
\end{equation}

\subsection{The Mutual Information Representation}
\label{app:proof-lemma}

We first identify what $\mathcal{I}(v)$ measures, and the support restriction makes the answer exact. Because both entropy terms are computed on the same population $R_m(v)$, the arm weights sum to one, and Eq.~\ref{eq:mi-restate} is exactly mutual information under the empirical joint distribution of actions and outcomes.

\begin{lemma}[Empirical joint representation]
\label{lem:joint}
Let $\hat{P}$ be the empirical joint distribution of $(A_v,O)$ on $R_m(v)$, that is $\hat{P}(a,o)=|\{\tau\in R_a(v):\omega(\tau)=o\}|/N$ for $a\in\mathcal{A}_m(v)$. Then $\mathcal{I}(v)=I_{\hat{P}}(A_v;O)$, the mutual information of $\hat{P}$. Consequently
\begin{equation}
\mathcal{I}(v)=\hat{H}(O)-\hat{H}(O\mid A_v)=\hat{H}(A_v)-\hat{H}(A_v\mid O),
\label{eq:sym}
\end{equation}
and $0\le \mathcal{I}(v)\le\min\{\hat{H}(A_v),\hat{H}(O)\}$.
\end{lemma}

\begin{proof}
The marginal of $\hat{P}$ on $A_v$ is $\hat{P}(a)=n_a/N$, and its conditional distribution given $A_v{=}a$ is the empirical distribution of $\omega$ on $R_a(v)$. Substituting these distributions into Eq.~\ref{eq:mi-restate} gives $\mathcal{I}(v)=I_{\hat{P}}(A_v;O)$. The identities and bounds then follow from standard properties of mutual information. Because $\hat{P}$ is fixed for each realization, no assumption about the sampling process is required.
\end{proof}

We can therefore reason directly about the two variables that matter: the action taken at $v$ and the outcome reached after it. The key question is whether finite sampling gives us enough separation to identify consequential decisions.

\subsection{Finite-Sample Separation}
\label{app:proof-prop}

We next show that the support floor creates such a separation, and it is the only place where the rollout budget $k$ enters the bound.

\begin{lemma}[Arm-entropy floor]
\label{lem:floor}
If $v$ is scorable then $\hat{H}(A_v)\ge H_2(m/N)\ge H_2(m/k)$, with equality in the first bound when $|\mathcal{A}_m(v)|=2$ and the arm counts are $(m,N-m)$.
\end{lemma}

\begin{proof}
Write $q_a=n_a/N$. Scorability gives at least two arms with $q_a\ge m/N$. Let $a^\star$ be an arm of minimum count, so $q_{a^\star}\in[m/N,1/2]$. By the grouping property of entropy,
\[
\hat{H}(A_v)=H_2(q_{a^\star})+(1-q_{a^\star})\,\hat{H}\!\left(\tfrac{q_a}{1-q_{a^\star}}\Big|_{a\ne a^\star}\right)\;\ge\;H_2(q_{a^\star})\;\ge\;H_2(m/N),
\]
using that $H_2$ is non-decreasing on $[0,1/2]$. Equality requires the residual term to vanish with $q_{a^\star}=m/N$, that is exactly two arms with counts $(m,N-m)$.

Scorability forces $N\ge 2m$, and $N\le k$, so $m/k\le m/N\le 1/2$ and monotonicity of $H_2$ on $[0,1/2]$ gives $H_2(m/N)\ge H_2(m/k)$.
\end{proof}

We can now ask whether the score exactly separates the two decision types that matter most. It does, provided the relevant arms are observed often enough.

\begin{proof}[Proof of Proposition~\ref{prop:sep}]
\emph{Clause 1, outcome-constant nodes.} If $\omega$ is constant on $R_m(v)$, then $\hat{H}(O)=0$, and Lemma~\ref{lem:joint} gives $0\le\mathcal{I}(v)\le\hat{H}(O)=0$.

\emph{Clause 2, outcome-separating nodes.} If the arms have disjoint outcome supports, then each observed outcome identifies a unique arm, so $\hat{H}(A_v\mid O)=0$. By Eq.~\ref{eq:sym},
\[
\mathcal{I}(v)=\hat{H}(A_v)-\hat{H}(A_v\mid O)=\hat{H}(A_v)\ge H_2(m/k),
\]
where the last inequality is Lemma~\ref{lem:floor}.

\emph{Separation.} For any $t\in\bigl(0,H_2(m/k)\bigr)$, outcome-constant nodes have $\mathcal{I}(v)=0<t$ and outcome-separating nodes have $\mathcal{I}(v)\ge H_2(m/k)>t$, so thresholding at $t$ separates the two classes without error. At $m{=}2$ and $k{=}8$, $H_2(2/8)=H_2(0.25)=2-\tfrac{3}{4}\log_2 3\approx 0.8113$, giving the admissible interval $(0,0.811)$ quoted in the proposition.
\end{proof}

The result requires no asymptotics. We obtain zero information when the outcome is constant, we obtain the full arm entropy when the outcome identifies the arm, and the support floor creates a finite-sample gap between the two. We also do not require each arm to produce a single outcome: an arm may produce several outcomes as long as no outcome is shared across arms. What remains is the ground between the extremes, where supports overlap.

\subsection{Partial Identifiability}
\label{app:proof-partial}

We rarely expect decisions to be perfectly constant or perfectly separating: arms share some outcomes while differing on others. We can characterize this intermediate regime directly through mutual information.

\begin{proposition}[Component-wise partial identifiability]
\label{prop:partial}
For a scorable node $v$, define the \emph{resolved fraction} $\rho(v)=\mathcal{I}(v)/\hat{H}(A_v)$ when $\hat{H}(A_v)>0$. Then
\begin{equation}
\rho(v)\in[0,1],\qquad \rho(v)=1-\frac{\hat{H}(A_v\mid O)}{\hat{H}(A_v)},
\label{eq:rho}
\end{equation}
so $\rho(v)=1$ exactly when $O$ determines the arm, recovering the outcome-separating case, and $\rho(v)=0$ exactly when $O$ carries no information about the arm, which includes the outcome-constant case. Furthermore
\begin{equation}
\hat{H}(A_v)-\mathcal{I}(v)=\hat{H}(A_v\mid O)\le\log_2\bigl|\mathcal{A}_m(v)\bigr|,
\label{eq:shortfall}
\end{equation}
and the shortfall of $\mathcal{I}(v)$ below $\hat{H}(O)$ is the average within-arm outcome entropy, $\hat{H}(O)-\mathcal{I}(v)=\sum_a (n_a/N)\hat{H}(O\mid A{=}a)$.
\end{proposition}

\begin{proof}
By Lemma~\ref{lem:joint}, $0\le\mathcal{I}(v)\le\hat{H}(A_v)$, giving $\rho(v)\in[0,1]$, and Eq.~\ref{eq:rho} follows by dividing $\mathcal{I}(v)=\hat{H}(A_v)-\hat{H}(A_v\mid O)$ by $\hat{H}(A_v)$. Thus $\rho(v)=1$ iff $\hat{H}(A_v\mid O)=0$, meaning each observed outcome co-occurs with exactly one arm, while $\rho(v)=0$ iff $A_v$ and $O$ are independent under $\hat{P}$, which an outcome-constant node satisfies because $O$ is degenerate. Equation~\ref{eq:shortfall} follows from Eq.~\ref{eq:sym} and $\hat{H}(A_v\mid O)\le\hat{H}(A_v)\le\log_2|\mathcal{A}_m(v)|$, and the final claim is Eq.~\ref{eq:mi-restate} rearranged.
\end{proof}

Equation~\ref{eq:rho} gives the interpretation: $\hat{H}(A_v)$ measures how much the agent's decisions vary, and $\mathcal{I}(v)$ measures the portion of that variation the observed outcome resolves. Arms reaching $\{x,y\}$ and $\{z,w\}$ at balanced counts give $\mathcal{I}(v)=\hat{H}(A_v)=1$ bit, whereas arms reaching $\{x,y\}$ and $\{y,z\}$ share the class $y$ and give $\mathcal{I}(v)=0.5$. SynAct therefore spends its acquisition budget on the component of decision variation the outcome actually resolves, and the digest exposes the corresponding divergent end states to the expert.

\subsection{Recovery Under Repeated Sampling}
\label{app:proof-recovery}

The separation result is conditional on scorability. The natural question is how often finite sampling recovers the relevant arms. A consequential arm that appears only once is invisible to the score however different its outcome, and we quantify this coverage directly.

\begin{proposition}[Arm recovery under repeated sampling]
\label{prop:recovery}
Fix a node $v$ and suppose each rollout is independent and reaches $v$ and takes arm $a$ with probability $p_{v,a}$. Then $N_{v,a}\sim\mathrm{Binomial}(k,p_{v,a})$ and
\begin{equation}
\Pr[N_{v,a}\ge m]=1-\sum_{j=0}^{m-1}\binom{k}{j}p_{v,a}^{\,j}(1-p_{v,a})^{k-j},
\label{eq:recovery}
\end{equation}
which at $m{=}2$ is $1-(1-p_{v,a})^k-kp_{v,a}(1-p_{v,a})^{k-1}$. If $r$ arms each satisfy $p_{v,a}\ge p_{\min}$, a union bound gives $\Pr[\text{all }r\text{ arms clear the floor}]\ge 1-r\bigl[(1-p_{\min})^k+kp_{\min}(1-p_{\min})^{k-1}\bigr]$, and $v$ is scorable once any two arms clear it.
\end{proposition}

\begin{proof}
Each rollout contributes an independent Bernoulli$(p_{v,a})$ trial, giving $N_{v,a}\sim\mathrm{Binomial}(k,p_{v,a})$ and Eq.~\ref{eq:recovery}. At $m{=}2$ the failure probability $(1-p_{v,a})^k+kp_{v,a}(1-p_{v,a})^{k-1}$ is non-increasing in $p_{v,a}$, and a union bound over the $r$ arms at $p_{v,a}\ge p_{\min}$ gives the stated bound. Any two retained arms make $v$ scorable.
\end{proof}

Table~\ref{tab:recovery} evaluates Eq.~\ref{eq:recovery} at the default $k{=}8$, $m{=}2$.

\begin{table}[t]
\centering
\small
\caption{\textbf{Arm recovery at the default budget} $k{=}8$, $m{=}2$, from Eq.~\ref{eq:recovery}. An arm the agent takes in half its rollouts is retained with probability $0.965$, whereas one taken a tenth of the time is retained with probability $0.187$. The support floor therefore selects reproducible decision components and discards singleton behavior.}
\label{tab:recovery}
\begin{tabular}{@{}lrrrrrr@{}}
\toprule
$p_{v,a}$ & 0.10 & 0.20 & 0.25 & 0.30 & 0.40 & 0.50 \\
$\Pr[N_{v,a}\ge2]$ & 0.187 & 0.497 & 0.633 & 0.745 & 0.894 & 0.965 \\
\bottomrule
\end{tabular}
\end{table}

The support floor therefore trades coverage for reproducibility. We do not claim $k{=}8$ recovers every consequential decision: it recovers the portion of behavior supported by repeated evidence, which is also the portion a written convention can address. Raising $k$ widens this coverage at linear rollout cost (Appendix~\ref{app:budget}).

\subsection{Visibility Beyond Success}
\label{app:proof-allfail}

We also ask whether the signal survives when success supplies no information. A natural alternative is to define $O$ by success, and this works when some attempts succeed, with a sharp blind spot: when every attempt fails, every decision looks identical under the success projection. Outcome information carries no such limitation.

\begin{proposition}[Beyond-capability visibility]
\label{prop:allfail}
Let $v$ be scorable on a task where no rollout succeeds, so $\mathrm{succ}(\tau)=0$ for every $\tau\in R_m(v)$. Then the success-based score is identically zero, $\mathcal{I}_{\mathrm{succ}}(v)=0$. If in addition two arms $a,a'\in\mathcal{A}_m(v)$ have differing empirical outcome distributions, $\hat{P}(O\mid A{=}a)\ne\hat{P}(O\mid A{=}a')$, then $\mathcal{I}(v)>0$.
\end{proposition}

\begin{proof}
For the first claim, $\mathrm{succ}$ is constant on $R_m(v)$, so Clause 1 of Proposition~\ref{prop:sep} gives $\mathcal{I}_{\mathrm{succ}}(v)=0$ for every realization. For the second, Lemma~\ref{lem:joint} identifies $\mathcal{I}(v)$ with $I_{\hat{P}}(A_v;O)$, which vanishes iff $A_v$ and $O$ are independent under $\hat{P}$. Independence would force $\hat{P}(O\mid A{=}a)=\hat{P}(O\mid A{=}a')$, contradicting the hypothesis, so $\mathcal{I}(v)>0$.
\end{proof}

The distinction is structural rather than empirical: success-based acquisition is blind on all-fail tasks by construction, whereas outcome information stays positive whenever different arms lead to distinguishable end states. This is the mechanism behind the coverage gap measured in \S\ref{sec:res-signal}, including the Gemma-12B Bash setting where success-MI authors no rule at all.

\subsection{Finite-Sample Ranking}
\label{app:proof-null}

The two extremes admit exact separation, and SynAct does not threshold decisions into classes. We rank them. We therefore ask what finite-sample estimation does to that ranking, and we answer under an exact null: take two arms with each run assigned uniformly, and $\omega$ drawn uniformly over $c$ classes independently of the arm, so $I(A_v;O)=0$ for the underlying law. Enumerating all $2^k c^k$ realizations and conditioning on scorability gives Table~\ref{tab:null}.

\begin{table}[t]
\centering
\small
\caption{\textbf{Exact null behavior of the plug-in score} when $O$ is independent of $A_v$, so the true mutual information is $0$. Computed by complete enumeration over all realizations, conditional on the node being scorable. $P_{\text{sc}}$ is the probability of scorability. The estimator is upward-biased, and the bias grows with the number of outcome classes $c$, which is why we rank by $\mathcal{I}(v)$ rather than threshold it at $0$.}
\label{tab:null}
\begin{tabular}{@{}lrrrr@{}}
\toprule
& \multicolumn{4}{c}{$k$} \\
\cmidrule(l){2-5}
Quantity ($m{=}2$, $c{=}2$) & 4 & 6 & 8 & 10 \\
\midrule
$P_{\text{sc}}$                              & 0.375 & 0.781 & 0.930 & 0.979 \\
$\mathbb{E}[\mathcal{I}(v)\mid\text{scorable}]$ & 0.281 & 0.177 & 0.124 & 0.094 \\
$\Pr[\mathcal{I}(v)\ge 0.3\mid\text{scorable}]$ & 0.625 & 0.219 & 0.159 & 0.063 \\
\midrule
$\mathbb{E}[\mathcal{I}(v)\mid\text{scorable}]$, $c{=}3$ & 0.463 & 0.328 & 0.247 & 0.195 \\
$\mathbb{E}[\mathcal{I}(v)\mid\text{scorable}]$, $c{=}4$ & 0.574 & 0.435 & 0.346 & 0.284 \\
\bottomrule
\end{tabular}
\end{table}

The plug-in score is therefore not a zero-centered statistic at finite $k$. The bias decays roughly as $1/k$ but remains non-negligible at $k{=}8$: a node with no true dependence scores $0.124$ bits on average. The bias also grows with the number of outcome classes, reaching $0.346$ bits at $c{=}4$, motivating the coarse per-benchmark projections of Appendix~\ref{app:outcome}. Finally, the null reaches the maximum $1.0$ bits with probability $0.002$: no threshold strictly below the maximum eliminates false positives in this regime, and Proposition~\ref{prop:sep} is correspondingly restricted to the two extremes.

Ranking is more useful than thresholding here. Model the arm-outcome dependence as a binary symmetric channel in which the outcome class matches the arm except with flip probability $\varepsilon$, so the true mutual information is $1-H_2(\varepsilon)$ bits, and compare an independent draw of a dependent node against an independent draw of a null node, each conditioned on its own scorability at $k{=}8$, $m{=}2$. The dependent node receives the strictly higher score with probability $0.995$ at $\varepsilon{=}0$, $0.882$ at $\varepsilon{=}0.125$, $0.692$ at $\varepsilon{=}0.25$, and $0.517$ at $\varepsilon{=}0.375$, where the true information is $0.046$ bits and the comparison approaches chance. The ranking therefore reliably lifts strongly separating nodes above the null at the default budget and degrades gracefully toward chance as the true dependence vanishes.

These results motivate the way SynAct consumes the statistic, as an acquisition score and never as a hypothesis test: selecting the top $B$ nodes compares scores against one another, whereas thresholding at zero would admit the entire null. The bias does vary with the support sizes, at the order of the Miller--Madow correction $(|\mathcal{A}_m(v)|-1)(|\mathcal{O}|-1)/(2N\ln 2)$. A general ranking-recovery guarantee would therefore need a separation condition accounting for this cross-node spread, and we do not claim one. The practical mitigation is a coarse per-benchmark $\omega$, which holds $|\mathcal{O}|$ comparable across nodes by construction. The scorability row also explains the default budget: at $k{=}8$ roughly $93\%$ of two-arm nodes are scorable, against $38\%$ at $k{=}4$ (Appendix~\ref{app:budget}).

\subsection{Assumptions and Scope}
\label{app:scope}
\label{app:proof-scope}

We have so far established statistical identifiability, and neither causal effect nor repairability. The score measures association under the agent's sampling distribution: a hidden variable can induce a high score without the action being responsible, and an outcome-separating decision need not be repairable by a written convention. We use ``identifiability'' narrowly, for whether finite-sample rollouts separate outcome-constant from outcome-separating decisions, a weaker sense than the recovery of latent factors up to ambiguity in representation learning~\citep{liu2026disentangled}, with the intermediate regime characterized by Proposition~\ref{prop:partial} and Appendix~\ref{app:proof-null}.

Two assumptions are load-bearing, the clauses of Assumption~\ref{as:ident}. Determinism is not one of them.

\paragraph{State identifiability.}
We require $\phi$ to distinguish distinct situations without making recurring ones unnecessarily unique. If $\phi$ aliases distinct situations, as under partial observability, the arms pool incomparable decisions and $\mathcal{I}(v)$ measures a mixture. If $\phi$ is over-discriminative, reading timestamps or continuous values, groups become singletons and fail the support floor. We therefore need sufficient recurrence without aliasing, and a learned representation could relax the exact-match form we use.

\paragraph{Outcome identifiability.}
We require $\omega$ to carry no encoding of the branching action: otherwise distinct arms are outcome-disjoint by construction and $\mathcal{I}(v)=\hat{H}(A_v)$, reducing the score to action entropy. We therefore restrict $\omega$ to the terminal state, ensuring that any dependence of $O$ on $A_v$ is mediated by where the run ended. A projection reading the actions a run performed violates this requirement, and one reading the resulting environment state satisfies it.

\paragraph{Determinism.}
We require determinism for neither condition and use it only to simplify implementation: equivalent situations produce identical observations, which lets $\phi$ use exact matching, and retrying without the injected rule reproduces the original failure, giving the gated protocol of \S\ref{sec:gate} its control arm for free. Under stochastic dynamics we would instead use a resample-only control to separate repair from environmental randomness, and the stochasticity making $\mathcal{I}(v)$ informative in our experiments already comes from policy sampling at temperature $0.7$ rather than from the environment.

For a new benchmark, we therefore need to redesign only $\phi$ and $\omega$: $\phi$ must induce recurring decision contexts, and $\omega$ must preserve outcome distinctions without encoding the branching action.

\section{Detailed Experimental Setup}
\label{app:setup}

We provide the configuration needed to reproduce Table~\ref{tab:main}, complementing the summary in \S\ref{sec:setup}.

\subsection{Benchmarks, Splits, and Step Limits}
\label{app:setup-bench}

We summarize the per-benchmark configuration in Table~\ref{tab:setup-bench}. We use exact-match state keys because repeating an action sequence reproduces the observation sequence closely enough to group equivalent situations without a similarity threshold, as Assumption~\ref{as:ident} requires. The $\tau$-bench user simulator is the least controlled case because its replies are model-generated, and Appendix~\ref{app:grouping} reports the resulting grouping statistics. We fix the split seed to 42 so that train and test stay disjoint and reproducible across runs.

\begin{table}[h]
\centering
\small
\setlength{\tabcolsep}{4pt}
\caption{Per-benchmark configuration. \emph{Steps} is the hard cap on agent actions per rollout. \emph{Workers} is the concurrency cap used during evaluation, which is 1 for the Docker-backed shell environment because that environment is not thread-safe.}
\label{tab:setup-bench}
\begin{tabular}{@{}lrrrrl@{}}
\toprule
Benchmark & $n_{\text{train}}$ & $n_{\text{test}}$ & Steps & Workers & Split rule \\
\midrule
InterCode-SQL       & 50 & 150 & 10 & 8 & deterministic shuffle, first 50 train, rest test \\
TextCraft           & 50 & 59  & 20 & 8 & depth-3 recipe seeds shuffled, split in half \\
InterCode-Bash      & 20 & 40  & 10 & 1 & \texttt{fs\_1} shuffled, first 20 train, rest test \\
$\tau$-bench retail & 50 & 115 & 20 & 4 & benchmark's own train/test task split \\
\bottomrule
\end{tabular}
\end{table}

\paragraph{InterCode-SQL.}
We evaluate text-to-SQL on the Spider development set~\citep{yu2018spider} through the InterCode harness~\citep{yang2023intercode} against a local MySQL instance. The agent issues SQL statements, observes result rows or errors, and submits a final query. We count an episode as successful when its result set matches the gold query's. The agent sees the natural-language question plus the table-to-column schema.

\paragraph{TextCraft.}
We evaluate compositional crafting on TextCraft~\citep{prasad2024adapt} at recipe depth 3, meaning the target item requires a three-level crafting tree. The agent issues \texttt{craft} and \texttt{get} commands and observes inventory changes. We count an episode as successful when the target item is present at the end. The prompt lists the full recipe menu and states that only raw leaf items may be obtained with \texttt{get}.

\paragraph{InterCode-Bash.}
We evaluate shell scripting on nl2bash~\citep{lin2018nl2bash} through InterCode, running each task in a Docker container. We count an episode as successful when the submitted command's stdout matches the reference command's output exactly, which makes this the most format-sensitive of the four. Because the container environment is not thread-safe, we run all Bash evaluation single-worker.

\paragraph{$\tau$-bench retail.}
We evaluate tool-agent-user interaction in the retail domain of $\tau$-bench~\citep{yao2024taubench}. The agent calls typed tools against a mock database while conversing with a simulated user, and we count an episode as successful when the final database state is correct and the domain policy wiki is respected. We run the user simulator at temperature 0 so that episodes are reproducible. This is the only benchmark of the four where an action is a structured tool call rather than a text command, and the only one requiring a native tool-calling interface, which is why we do not evaluate the two Gemma models here.

\subsection{Models and Serving}
\label{app:setup-models}

We serve all agent models through Amazon Bedrock: Qwen3-32B~\citep{qwen3}, Gemma-3-12B and Gemma-3-4B~\citep{gemma3}, and Ministral-3-3B~\citep{ministral3}. We use Claude Opus 4.6~\citep{claudeopus} as the expert, called with temperature $0$, a $4000$-token output cap, and extended reasoning enabled. We sample acquisition rollouts at temperature $0.7$ and evaluate at temperature $0$ (greedy), which makes the two attempts of the gated protocol differ only in whether the rule is present. We cap per-step output at $1024$ tokens.

\paragraph{Reasoning mode.}
Qwen3-32B runs in thinking mode, and the trajectory text passed to the gate excludes the reasoning trace so that the gate judges actions rather than deliberation. The other three models have no separate reasoning channel.

\paragraph{Tool routing.}
All four models run on all four benchmarks, with $\tau$-bench tool calls routed natively through each model's tool-calling interface.

\subsection{Acquisition Hyperparameters}
\label{app:setup-hparam}

We use $k{=}8$ rollouts per training task, a minimum arm size of $m{=}2$, and acquire one decision node per task, the highest-scoring one. We vary $k \in \{2,4,6\}$ and the expert budget $B \in \{10\%,25\%,50\%\}$ of MI-ranked decisions in the corresponding ablations, holding everything else fixed. For the online variant, we use a consistency threshold $\theta{=}0.6$ to decide whether a task needs intervention at all, and we threshold $\mathcal{I}(v^\star)$ on the remainder, sweeping its operating range so that the expert-consultation rate spans the reported curve.

We cache rollouts to disk per (benchmark, model, task), letting us resume sweeps without re-sampling and recompute acquisition from a fixed rollout set. Each ablation re-uses the same cached rollouts wherever the setting permits, isolating the acquisition function from sampling variation.

\subsection{Error Gate}
\label{app:setup-gate}

Because we implement the gate as a single call to the \emph{agent} model rather than to the expert, gating adds no expert cost, though it does add one agent call of latency to each failed attempt. We pass it the task description, the first $1500$ characters of the failing trajectory, the antipattern sentence, and the rule, and it answers \texttt{FIRES} or \texttt{SKIP}. We default malformed or failed responses to \texttt{SKIP}: an infrastructure error can only withhold the rule rather than inject it spuriously. We record $n_{\text{fired}}$, the number of tasks reaching the gated retry, alongside each accuracy number.

\subsection{Baseline Configurations}
\label{app:setup-baselines}

We give every baseline the same expert model SynAct uses (\S\ref{sec:setup}), with the expert performing the method's own distillation step. Table~\ref{tab:main} thus compares where expert attention is allocated while holding expert strength fixed. We configure each method by its own protocol: ACE~\citep{zhang2026ace} runs one epoch of reflect-and-curate over the training tasks and injects the resulting playbook always-on, ReasoningBank~\citep{ouyang2026reasoningbank} distils one memory item per training task and retrieves the top $k{=}3$ at test time, AWM~\citep{wang2025awm} induces workflows from successful training trajectories and injects them always-on, and AWM-iter repeats the induction for ten rounds, matching our iteration count. Consequently the expert-turn counts of Table~\ref{tab:cost} scale with the training set for the baselines and with the iteration count for SynAct. We tune only setup parameters, namely training-set size, rollout count, epoch count, retrieval depth, and token budget, and leave each method otherwise unmodified. The self-supervised variants of Appendix~\ref{app:selfjudge}, ACE (val-select) and ReasoningBank (top-1), follow the earlier protocol of Appendix~\ref{app:inflight} and use no expert.

\subsection{Measurement Protocol}
\label{app:setup-protocol}

We report each setting as the mean of 3 independent repeats over the full test set. Each repeat re-runs both attempts of the gated protocol, including the unaided first attempt, re-measuring base accuracy within every repeat and at every iteration. This lets us estimate measurement noise from the experiment itself instead of assuming it: base accuracy moves $3.4$--$8.5$pp across iterations of the same setting. We therefore pool the base over completed iterations when reporting a base column and treat adjacent-iteration differences of that magnitude as indistinguishable at 3 repeats.

We generate per-setting numbers with a script that reads the run directory and refuses to print an entry lacking all 3 repeats: a partially-evaluated result cannot enter a table, and no number is transcribed by hand. We also record paired flip counts (win$\rightarrow$lose and lose$\rightarrow$win) per setting, because aggregate accuracy hides churn in which a method breaks as many tasks as it fixes.

\section{State-Keyed Grouping}
\label{app:grouping}

Acquisition compares arms only when they represent the same decision, and the grouping key determines which actions SynAct treats as alternatives from a shared state. We describe the key used in our experiments and the approximation required by our cached rollout format.

\paragraph{The key.}
We define a decision node as a maximal set of (run, step) pairs sharing a state key. The intended key is a normalized hash of the observation at $s$, canonicalized per benchmark by lower-casing, collapsing whitespace, and dropping volatile fields such as step counters, as constrained by Assumption~\ref{as:ident}. Our cached rollouts, however, record actions rather than full observations at every step, and the implementation keys each depth by the ordered action prefix of the run. An order-invariant alternative keys on the prefix multiset ${!{a_0,\dots,a_{d-1}}!}$, additionally merging runs that reach one state through different action orders. We audit this alternative below and find that its extra merges are usually illegitimate, precisely where they occur most often.

\paragraph{Why not index by position.}
A position index treats step $d$ of every run as the same decision even when runs have reached different states, which happens routinely once a trajectory takes an extra exploratory action. Recovering decision context from position alone therefore requires the action history to determine the state, whereas state keying makes the context explicit. The tradeoff runs in the other direction as well: an over-discriminative $\phi$ fragments a node into singletons, which the support floor then discards. The key thus trades coverage against the risk of aliasing distinct decisions (Appendix~\ref{app:scope}).

\paragraph{Collision audit.}
Does the ordered-prefix key alias distinct decisions, or split decisions it should merge? We re-group the cached $k{=}8$ training rollouts from all 16 settings under the multiset key and compare the resulting groups. Merges are structurally common in the flexible-order benchmarks: $1{,}943$ of $6{,}536$ depth groups merge more than one ordered prefix. Yet the acquired top-1 set barely changes: 28 of 565 routed tasks change ($5.0\%$), and the per-setting overlap is at least $0.92$ in 15 of 16 settings. To determine whether these extra merges are genuine, we replay the merged branches to their reached states and ask Claude Opus 4.6 whether they face the same decision. On SQL, every sampled merge is legitimate (0 of 19 judged distinct). On TextCraft, however, $90\%$ conflate genuinely distinct decisions, different crafting sub-goals reached through different action multisets, with $\tau$-bench between the two at $17.5\%$. The ordered-prefix key is therefore the correct default: where action order carries state, the order-invariant merge creates exactly the aliasing that Assumption~\ref{as:ident} rules out.

\section{Outcome Projections \texorpdfstring{$\omega$}{omega}}
\label{app:outcome}

The acquisition signal needs an outcome representation that distinguishes the consequences of different actions without consulting a gold answer. We therefore project each run's terminal state into a benchmark-specific outcome class using only its own trajectory; two runs share a class when they end in the same projected state.

\begin{itemize}
\itemsep0.1em
\item \textbf{InterCode-SQL}: the last executed \texttt{SELECT}/\texttt{WITH} statement, lower-cased with whitespace collapsed. Two runs submitting the same query share an outcome.
\item \textbf{TextCraft}: the sorted set of \texttt{craft}/\texttt{get} targets the run ended with, representing what it actually made, computed order-invariantly.
\item \textbf{InterCode-Bash}: the final submitted command, normalized and excluding pure navigation (\texttt{cd}, \texttt{ls}). Bash success is determined by the command's output, making the command the closest no-gold proxy for the terminal state.
\item \textbf{$\tau$-bench retail}: the multiset of write tool calls made by the run, computed order-invariantly, since $\tau$-bench success is determined by the database end state.
\end{itemize}

\paragraph{How sensitive is acquisition to the choice of $\omega$?}
We coarsen and refine each projection and rescore the cached rollouts. Refining is nearly rank-preserving on SQL (Spearman $0.98$ to $1.0$ against the default) and mildly destabilizing on TextCraft, whereas coarsening destabilizes both directions of the pipeline: on SQL, it collapses distinct queries into one class and shrinks the routed set from $43$ to $18$ tasks for Qwen; on TextCraft, it scrambles the ranking (Spearman $0.46$ to $0.52$). Rank movement does not translate into end-to-end damage, however. Authoring one iteration under the coarse projection is competitive-to-better on all four re-run settings, and on TextCraft the coarse variant lifts more than the default ($+17.2$ vs.\ $+13.8$pp for Qwen, $+12.1$ vs.\ $+6.9$pp for Gemma-12B) at matched coverage. Grouping more divergent endings per mode lets the expert surface an additional valid convention. The default projection thus sits in a usable middle, and the authored rule tolerates substantial variation in granularity.

\paragraph{Sensitivity to projection granularity.}
We further examine how the granularity of $\omega$ affects acquisition by rescoring the cached rollouts under a coarser projection (the table set touched on SQL, the count of craft actions on TextCraft) and a finer one (the raw un-normalized terminal query, the full ordered action sequence). Refining is nearly rank-preserving on SQL (Spearman $\rho \ge 0.98$ against the default) and mildly destabilizing on TextCraft ($\rho$ of $0.62$ to $0.84$), while coarsening moves rankings more, shrinking the routed set on SQL from 43 to 18 tasks and dropping $\rho$ to $0.46$ on TextCraft. Yet these ranking changes do not translate into end-to-end damage: authoring and evaluating under the coarse $\omega$ matches the default on SQL and improves TextCraft for both swept models ($+17.2$ vs.\ $+13.8$pp for Qwen, $+12.1$ vs.\ $+6.9$pp for Gemma-12B at one iteration). Grouping more divergent endings into the same class gives the expert more evidence per mode. The default $\omega$ therefore lies in a forgiving region, while coarser projections remain a defensible, cheaper setting for free-form action tasks.

\section{Worked Example}
\label{app:example}

What does the expert see after SynAct identifies a consequential decision? We trace one iteration on TextCraft with Qwen3-32B. We draw $k{=}8$ rollouts on each of the 50 training tasks, group their steps by state key, and score each node by Eq.~\ref{eq:mi}. Keeping the highest-scoring node per task yields 49 scorable tasks, including 20 from tasks where no rollout succeeded and which a success-based score would have discarded (Proposition~\ref{prop:allfail}).

\paragraph{The evidence.}
We compress each acquired node into one digest block holding the task goal, two arms whose end states differ, the diverging action and the trajectory that followed for each, and a \texttt{correct}/\texttt{wrong} tag from that arm's success flag. Acquisition itself reads no labels, and the tags enter only here. The digest contains qualitatively different kinds of evidence.
\begin{itemize}
  \itemsep0.15em
  \item A task where both arms succeed by different routes, one reaching \texttt{craft 2 lime dye} through \texttt{green dye} and the other through \texttt{bone meal}. The expert sees two valid behaviors rather than a single canonical route.
  \item A task where both arms fail differently: one issues \texttt{craft 1 tripwire hook} before its inputs exist, and the other loops \texttt{craft 4 stick using 2 planks} six times.
  \item An \texttt{ALL-FAIL} task carrying two distinct wrong end states, one repeating \texttt{craft 1 granite using 1 diorite, 1 quartz} without acquiring the inputs. No arm is labeled correct, and the divergent failures are the only evidence available. This is precisely the case in which the success bit is silent.
\end{itemize}

\paragraph{The rule the expert returns.}
From all 49 blocks in one call the expert produces the following convention, reproduced verbatim as the agent receives it.
\begin{promptbox}[Convention authored by the expert]
- Before issuing any `craft' command, verify that ALL required input materials are already in your inventory; if not, first `get' or craft the missing prerequisites.\\
- At the start of a task, trace the full recipe tree to leaf (raw) materials and compute the TOTAL quantity of each needed; acquire all base materials before beginning the crafting chain.\\
- Never repeat the same craft or get command more than twice in a row; if it fails or yields no progress, re-check inventory and revise your plan.\\
- Work bottom-up: craft intermediate ingredients first, then combine them into the final product.
\end{promptbox}
\begin{promptbox}[Antipattern that gates it]
Attempting to craft items without first ensuring all input materials are present in inventory, leading to failed crafts, wasted steps, and infinite retry loops.
\end{promptbox}

\paragraph{What the expert abstracts.}
The four bullets cover distinct failure modes in the acquired evidence. The first addresses crafting before acquiring ingredients, the third addresses the repeated \texttt{craft 4 stick} loop, and the second and fourth generalize the two successful routes into a plan-then-execute ordering. Crucially, the convention names none of the concrete items that appeared in training and therefore transfers to recipe trees the training tasks never contained. A per-episode repair cannot produce this abstraction, since counting modes across blocks is what identifies which failures are systematic. Prepending these four bullets behind the gate lifts accuracy in this setting by $14.2$pp over base.

\paragraph{Application.}
The resulting convention is not applied blindly. At test time the task first runs without it at temperature $0$, and a successful attempt terminates immediately. After a failure, the gate reads the trajectory against the antipattern, and only when it fires do we prepend the convention and retry once.

\section{Rollout and Expert Budget Sensitivity}
\label{app:budget}

SynAct spends two distinct budgets: rollouts determine how much behavioral variation we observe, and the expert budget determines how much of that variation we ask the expert to explain. We vary each independently to measure how quickly the signal saturates.

\begin{table}[t]
\centering
\small
\setlength{\tabcolsep}{5pt}
\caption{\textbf{The two budget knobs.} Accuracy (\%, 3-repeat mean) when cutting rollout count $k$ ($k{=}8$ is the full method) and expert budget $B$, the top $B\%$ of MI-ranked decisions ($B{=}100\%$ is full). Each panel carries its own \emph{base}, pooled over that panel's settings because the gate re-measures the unaided agent in every run, and bold marks the best setting per row. In the mean-gain row, --- marks the base columns, which have no gain by definition. All 36 results are complete at 3 repeats.}
\label{tab:budget}
\begin{tabular}{@{}ll rrrr rrrr@{}}
\toprule
\multicolumn{2}{@{}l}{} & \multicolumn{4}{c}{Rollouts $k$} & \multicolumn{4}{c}{Budget $B$} \\
\cmidrule(lr){3-6}\cmidrule(l){7-10}
Bench & Model & base & 2 & 4 & 6 & base & 10\% & 25\% & 50\% \\
\midrule
\multirow{3}{*}{SQL}
 & Q    & 65.5 & 69.3 & 69.6 & \textbf{70.4} & 65.9 & 70.9 & \textbf{71.1} & 66.7 \\
 & M    & 51.6 & 63.6 & 63.8 & \textbf{66.0} & 49.9 & \textbf{64.4} & 62.4 & 61.6 \\
 & 12B  & 63.9 & 66.9 & 66.4 & \textbf{68.0} & 63.5 & 63.3 & \textbf{66.4} & 66.2 \\
\midrule
\multirow{3}{*}{TC}
 & Q    & 41.1 & 52.0 & 56.5 & \textbf{57.1} & 43.5 & 59.9 & \textbf{61.6} & 54.2 \\
 & M    & 12.2 & 16.9 & \textbf{18.6} & 11.9 & 12.2 & 16.4 & \textbf{16.9} & 14.7 \\
 & 12B  & 16.8 & 28.2 & 24.3 & \textbf{31.1} & 19.0 & 21.5 & 31.1 & \textbf{31.6} \\
\midrule
\multicolumn{2}{@{}l}{\emph{mean gain over base}} & --- & $+7.6$ & $+8.0$ & $\mathbf{+8.9}$ & --- & $+7.1$ & $\mathbf{+9.2}$ & $+6.8$ \\
\bottomrule
\end{tabular}
\end{table}

\paragraph{How many rollouts do we need?}
Equation~\ref{eq:mi} requires $m{=}2$ runs per arm before a node becomes scorable, but additional rollouts can expose arms that a smaller budget misses and make the estimate more stable. We therefore sweep $k \in \{2,4,6\}$ against the $k{=}8$ default (Table~\ref{tab:budget}, left). All six settings already beat their base at $k{=}2$, a $4\times$ cheaper collection, and this small budget often captures most of the gain available at $k{=}6$: Ministral SQL gains $+12.0$pp at $k{=}2$ against $+14.4$pp at $k{=}6$ over a pooled base of $51.6$. Accuracy rises from $k{=}2$ to $k{=}6$ in 5 of 6 settings, and the exception, Ministral TextCraft ($16.9/18.6/11.9$ on a $12.2$ base), is the noisiest setting in the grid. Since $k$ multiplies the dominant collection cost, these results favor a small rollout budget: coarse behavioral divergence carries most of the useful signal without requiring a precise probability estimate.

\paragraph{How much expert attention is enough?}
The ranking lets us truncate the expert budget without changing the acquisition rule: we expose only the top $B\%$ of MI-ranked decisions. The top decile already gains $+7.1$pp over base on average and the quarter-budget reaches $+9.2$pp, which on Qwen SQL gives $71.1$ against the full-budget $68.9$ of Table~\ref{tab:main}. The leading decisions therefore contain most of the dominant failure mode, making much of the remaining budget redundant.

More budget does not always help. At $B{=}50\%$ the average gain falls to $+6.8$pp, below the quarter-budget on 5 of 6 settings. A larger digest contains more heterogeneous failure modes while the expert still compresses the evidence into a single convention, and beyond a point the rule becomes less focused rather than more complete. The exception, Gemma-12B TextCraft at $31.6$ against $31.1$, is a $0.5$pp difference on a 59-task set and does not separate at 3 repeats.

The two sweeps share a pattern. We do not need exhaustive behavioral coverage to obtain useful feedback: a handful of rollouts exposes enough divergent behavior to rank decisions, and the highest-ranked decisions carry enough shared structure for the expert to form a useful convention. SynAct therefore spends compute on discovering where behavior diverges and spends expert attention on the most concentrated part of that divergence.

\paragraph{Where does the rollout budget actually bind?}
A full ten-iteration training sweep on TextCraft with Qwen sharpens the picture: $k{=}2$ degenerates completely, exposing no within-state divergence, surfacing zero branches, and authoring no rule at any iteration. In contrast, $k \in {4,8,16}$ all learn, with iteration-10 lifts of $+20.7$, $+27.6$, and $+12.1$pp, respectively, within single-repeat noise. The sharp dependence is therefore the $k{=}2$ cliff, not the exact budget above it, consistent with the recovery probabilities in Proposition~\ref{prop:recovery}: the signal needs enough rollouts to observe self-disagreement at all, then saturates quickly. Throttling the per-iteration expert budget to the top $25\%$ or $50\%$ of routed tasks over the same sweep retains most of the full-budget lift ($+20.7$ and $+17.2$ vs.\ $+27.6$pp).

\paragraph{Sensitivity to the support floor.}
Rescoring the cached rollouts at $m{=}1$ and $m{=}3$ shows that the default $m{=}2$ sits at a knee. Dropping to $m{=}1$ adds acquisitions in only 4 of 16 settings, almost all on the weakest model, whereas raising to $m{=}3$ costs coverage in 10 of 16, collapsing Ministral's TextCraft set from 37 to 9 tasks. Running $m{=}1$ authoring end to end on the two most affected settings never hurts and helps TextCraft with Ministral ($+5.2$ vs.\ $+1.7$pp at the same firing set), because several singleton arms carry verifiably correct recipe facts that the floor would otherwise discard. The floor is thus a safe default with a real coverage cost for weak agents, matching the coverage tradeoff in Proposition~\ref{prop:recovery}.

\paragraph{Does the acquisition ranking itself carry the gain?}
At one iteration under the final protocol, swapping the selection rule while holding the digest format, authoring call, and gate fixed yields statistically indistinguishable full-set gains across all 16 settings ($+6.6$ to $+7.6$pp), with per-setting winners spread across all four signals. The explanation is visible in the authored rules: when one failure mode dominates a benchmark, every selector over branching points surfaces it, and on SQL all four signals converge on the same convention. The value of outcome information therefore concentrates where the theory predicts: all-fail tasks that success-based selection cannot see (\S\ref{sec:res-signal}) and per-task routing at inference time (Appendix~\ref{app:signal-routing}). The end-to-end training loop is nevertheless robust to the precise ranking, lowering the cost of adopting the method.

\section{Coverage of Acquisition Signals}
\label{app:coverage}

How much of the gain comes from knowing \emph{where} to ask the expert? The clean test holds the rollouts, authoring call, and gate fixed and varies only the selection rule, comparing outcome information against random selection over decision nodes, action entropy, and a success-based score (Table~\ref{tab:signal}, under the single-iteration protocol of Appendix~\ref{app:inflight}). Outcome-MI is best in all 8 settings and all 23 pairwise comparisons, improving over base by $+12.5$pp on average against $+8.2$pp for random, $+6.9$pp for success-MI, and $+6.8$pp for entropy. Random selection is therefore a strong baseline, since a strong expert helps wherever it is pointed, and the additional margin is what pointing it at the right decisions buys. More strikingly, entropy trails random: disagreement alone spends the budget on cosmetic variation that does not separate outcomes, precisely the distinction Proposition~\ref{prop:sep} draws. Success-MI has a different failure mode: where no task contains both a winning and a losing rollout, it has nothing to acquire.

The gap lies in coverage as well as ranking quality. Accuracy in Table~\ref{tab:signal} is defined only where a signal produces evidence, and Clause 1 of Proposition~\ref{prop:sep} predicts where each signal goes silent. Success-MI requires \emph{both} a winning and a losing rollout, and on Gemma-12B Bash zero mixed contrasts occur, no rule is authored, and accuracy stays at the $19.2$ base. Outcome-MI carries no such requirement: accuracy in the same setting reaches $31.7$ ($+12.5$pp), and the signal acquires $1.9\times$ more decisions overall ($274$ against $146$ on one SQL model). This separates two properties of an acquisition signal: \emph{quality}, how well it ranks the evidence available to it, and \emph{coverage}, how often it produces evidence at all. Under a fixed expert budget the latter binds. A signal that goes silent on difficult tasks does not merely rank poorly: it leaves its budget unused exactly where feedback is most needed.

\section{Scope of the Acquisition-Function Comparison}
\label{app:inflight}

Table~\ref{tab:signal} gives the full acquisition-function sweep on InterCode-SQL and TextCraft. We restrict the comparison to these two benchmarks because holding the rollouts, the authoring call, and the gate fixed while varying only the selection rule costs one complete training run per rule per setting. The resulting comparison is broad across models and narrower across environments: the claim that outcome information beats random, entropy, and success-based selection rests on 8 settings across two benchmarks.

This scope also explains why the base rows differ from Table~\ref{tab:main}. The sweep predates the equal-expert protocol of \S\ref{sec:res-main}, leaving its absolute values incomparable across tables. Within each setting, however, the four selection rules share the same rollouts, expert, and gate, which isolates the acquisition function. The same provenance applies to the rollout and budget sensitivity grid of Table~\ref{tab:budget}. The headline training results, the iteration curves of Table~\ref{tab:curve}, and the routing results of Table~\ref{tab:online} all come from the final protocol.

\section{Why Self-Supervised Curation Saturates}
\label{app:selfjudge}
Can stronger curation close the gap without changing the source of supervision? We compare each baseline against a more selective variant under the single-iteration protocol of Appendix~\ref{app:inflight}: \textbf{ACE (val-select)} deploys its playbook only when it beats the empty playbook on a held-out split, and \textbf{ReasoningBank (top-1)} retrieves at most one memory. Neither changes the basic picture. ACE and ReasoningBank land within noise of base about as often as they improve on it, with a negative aggregate effect on the two weakest models on TextCraft, and the protective variants mostly trade failure for stagnation: ACE (val-select) often declines to deploy, and ReasoningBank (top-1) stays flat. The limiting factor is therefore deeper than curation, since a weak agent's traces record what it did without recording why it was wrong.

Only two settings favor a baseline, both on Gemma-3-4B: ACE on Bash ($23.3$ against $20.0$, a margin of $1.3$ tasks) and ReasoningBank (top-1) on TextCraft ($4.0$ against $2.8$), both near the floor where under one task separates the methods. The result points to a limitation of self-supervised experience distillation: selecting better memories cannot recover information the generating agent never possessed.

\section{Error Gate Ablation}
\label{app:gate-detail}

\paragraph{Gating against unconditional injection.}
Does the non-regression come from the authored rule or from the gate that decides when to apply it? Holding the rule fixed and changing only its application answers this directly. Under the single-iteration protocol of Appendix~\ref{app:inflight}, unconditional injection produces 32 to 36 win$\rightarrow$lose flips on SQL against $0$ when gated, and Table~\ref{tab:main} shows the same pattern in aggregate, with always-on baselines regressing on several settings while SynAct regresses on none. Flip counts expose the effect more clearly than accuracy alone, since equal numbers of repairs and regressions cancel in the aggregate. The gate thus provides deployment-time protection without a held-out validation split: when the unaided attempt succeeds, the rule is withheld and cannot undo that success.

\paragraph{Is the served gate faithful?}
The gate is a single call to the agent model, and we measure how well it approximates expert judgment. On sampled failing trajectories, we compare the served gate's firing decision against an independent Claude Opus 4.6 judgment of whether the failure exhibits the antipattern. On TextCraft, Bash, and $\tau$-bench, the gate reaches recall $1.0$, firing on every failure flagged by the expert judge, at precision $0.58$ to $0.86$. SQL is the conservative outlier, with recall $0.22$ at precision $0.50$, because its failures are heterogeneous and the per-item selector often abstains. Abstention is cheap, however: force-injecting the convention into gate-declined failures repairs only $15$ to $21\%$ of them: most skipped failures would not have been fixed by the rule anyway. Firings on unrepairable failures waste a retry, $48$ to $84\%$ of firings depending on the benchmark, but because the gate is consulted only after a failure, this waste costs compute rather than accuracy: win$\rightarrow$lose churn is $0$ in every setting.

\paragraph{Does gating alone explain the gap?}
A sharper decomposition redeploys each baseline's own Opus-authored library behind SynAct's gate, changing only the deployment protocol. Gating the ACE and AWM libraries recovers roughly half of SynAct's full-set gain ($+4.6$ and $+5.0$pp vs.\ $+8.2$ over the 12 swept settings), showing that selective application is a substantial transferable source of benefit. The remainder comes from the knowledge itself: SynAct's rule repairs $18\%$ of the failures it fires on, versus $11\%$ for either gated baseline library. On InterCode-Bash, where the baseline libraries rarely trigger the gate, targeted acquisition dominates ($+11.7$ vs.\ $+3.1$ and $+2.5$pp). Both components matter, but their relative contribution varies by benchmark.

\paragraph{Does the author's strength matter?}
Swapping the authoring expert from Claude Opus 4.6 to Claude Sonnet 4.6, while holding the digest, gate, and evaluation fixed, lowers the mean firing-set lift from $+8.8$ to $+4.9$pp over 8 settings, with Opus ahead or tied in 7. The gap concentrates on the hardest benchmark: on Bash with Qwen, Sonnet authored a convention prescribing exactly the buggy aggregation that the evidence warned against, and the gate never matched its antipattern to a failure, and the rule silently deactivated. Author quality therefore matters, and a bad rule fails safely through inaction rather than causing damage.

\section{Accuracy Across Acquisition Iterations}
\label{app:iter}

\paragraph{Does more feedback keep helping?}
A procedure that rewrites its own context on each pass can improve in two ways: later iterations may uncover new failure modes, but they may also compound mistakes from earlier ones. We therefore run ten iterations per setting, re-rolling the training set under the active rule, and evaluate each iteration's rule independently (Table~\ref{tab:curve}, Figure~\ref{fig:train-curves}). The resulting curves are not uniformly front-loaded. The first iteration already gives $+11.4$pp on Qwen TextCraft, $+15.8$pp on Ministral SQL, and $+15.4$pp on Qwen Bash, and it captures at least $70\%$ of the best-iteration gain in half the settings. Others continue to improve substantially: Qwen TextCraft rises from $+11.4$pp at iteration~1 to $+20.6$pp at iteration~6, and Qwen $\tau$-bench from $+10.3$pp to $+15.2$pp. The distinction is intuitive: one convention can repair a dominant failure mode, whereas several modes require several passes. Since each iteration costs one expert call, iteration count is therefore the more attractive knob for increasing expert compute.

\paragraph{Does iteration introduce drift?}
The curves show little evidence that successive rules accumulate damage. Across the $160$ iteration evaluations in Table~\ref{tab:curve}, exactly one falls below its own base, Gemma-12B Bash at iteration~3 ($23.3$ against a $23.5$ base), a $0.2$pp dip within noise on a 40-task set. The gate provides a simple explanation for this stability: when a task does not exhibit the antipattern, the rule is withheld, and an unhelpful convention cannot touch a solved attempt.

\begin{figure}[h]
\centering
\includegraphics[width=\linewidth]{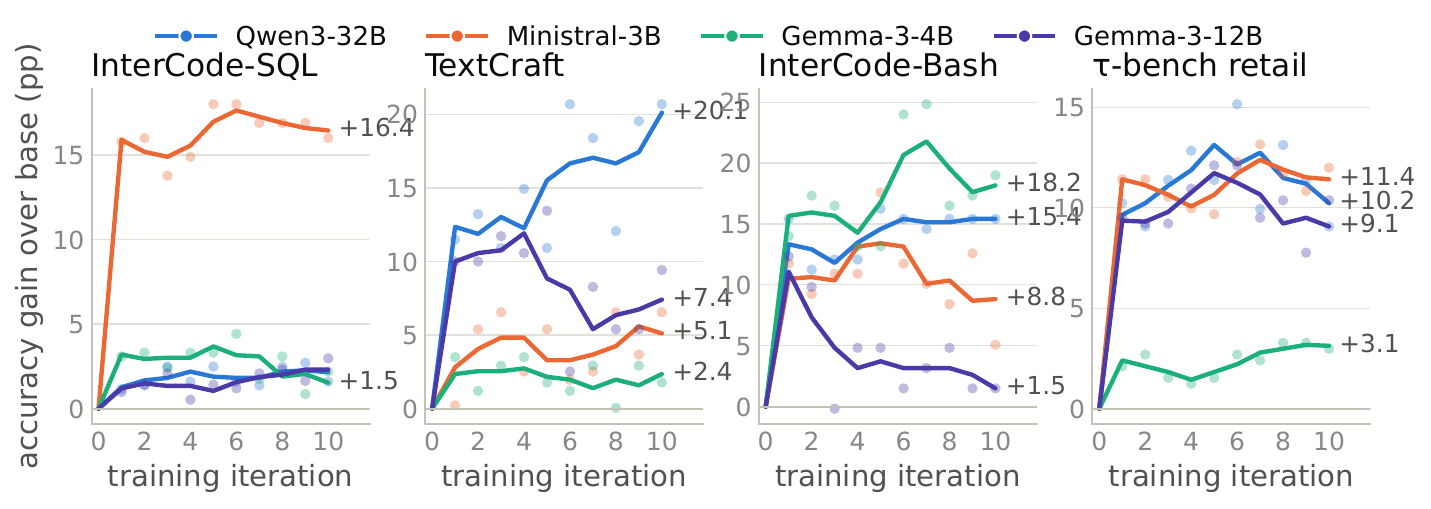}
\caption{\textbf{Gain over the re-measured greedy base per training iteration}, one panel per benchmark and one line per agent model. Solid lines are 3-iteration moving averages and faded dots are raw 3-repetition means.}
\label{fig:train-curves}
\end{figure}

\section{Full Per-Iteration and Per-Acquisition Tables}
\label{app:tables}

\paragraph{Statistical reporting.}
The test sets are small enough that a single task visibly moves a setting's accuracy: one task changes accuracy by $0.7$pp on InterCode-SQL ($n{=}150$), $1.7$pp on TextCraft ($n{=}58$), $2.5$pp on InterCode-Bash ($n{=}40$), and $0.9$pp on $\tau$-bench ($n{=}115$). We therefore flag near-floor comparisons where fewer than one task separates methods, treat each model--benchmark setting as the paired unit, and use exact one-sided sign tests across the 16 settings. SynAct at one iteration beats base in all 16 settings ($p=1.5\times 10^{-5}$), each individual baseline in 14--15 settings ($p\le2.1\times10^{-3}$), and the per-setting oracle in 14 settings ($p=2.1\times10^{-3}$). Routing shows the same pattern: SynAct-rules wins all 16 settings against base and cascade ($p=1.5\times10^{-5}$), 15 against SR(10) ($p=2.6\times10^{-4}$), 14 against the judge ($p=2.1\times10^{-3}$), and 12 against action entropy ($p=1.8\times10^{-2}$). Because sign tests use only win counts, they are conservative and avoid conflating the different scales of accuracy changes across benchmarks. Within each setting, exact McNemar tests on paired per-task outcomes provide the same picture. Across 63 single-realization comparisons of SynAct's best iteration against base and the three baselines, 31 are significant at $p<0.05$ and 37 have bootstrap $95\%$ confidence intervals strictly above zero. Mean advantages range from $+9.8$ to $+12.6$pp, with only three negative point estimates, none significant. The largest gains are likewise robust: Bash with Gemma-4B reaches $+35.0$pp against ACE, with 14 discordant wins versus 0 losses, while the strongest-base settings appropriately straddle zero.

\begin{table}[h]
\centering
\caption{\textbf{Training-time methods with an equal expert}, test accuracy (\%, 3-repeat mean), referenced from \S\ref{sec:res-main}. The shaded rows are ours. \emph{base} for SynAct is pooled over its complete iterations, and each baseline reports its own paired greedy base. SynAct appears after one training iteration and at its best iteration, with the best iteration in parentheses. Bold marks the best learning method per column and underline the second best.}
\label{tab:main}
\small
\setlength{\tabcolsep}{5pt}
\begin{tabular}{l r r r r}
\toprule
Method & Qwen3-32B & Ministral-3B & Gemma-4B & Gemma-12B \\
\midrule
\multicolumn{5}{@{}l}{\textit{InterCode-SQL ($n{=}150$)}}\\
base & 66.8 & 50.9 & 47.6 & 64.6 \\
ACE & 66.7 & 47.3 & 48.0 & 64.7 \\
ReasoningBank & 65.3 & 50.7 & 30.0 & \underline{67.3} \\
AWM & 66.7 & 56.7 & 34.0 & 65.3 \\
AWM-iter & 66.7 & 56.7 & 38.0 & 64.0 \\
\rowcolor{rowhl}
SynAct (1 iter) & \underline{68.0} & \underline{66.7} & \underline{50.7} & 65.6 \\
\rowcolor{rowhl}
SynAct (best iter) & \textbf{69.6} (9) & \textbf{68.9} (5) & \textbf{52.0} (6) & \textbf{67.6} (10) \\
\addlinespace[2pt]
\multicolumn{5}{@{}l}{\textit{TextCraft ($n{=}58$)}}\\
base & 46.6 & 13.0 & 2.8 & 19.3 \\
ACE & 55.2 & 15.5 & \underline{3.4} & 22.4 \\
ReasoningBank & 56.9 & 15.5 & \underline{3.4} & 15.5 \\
AWM & 44.8 & \textbf{20.7} & \underline{3.4} & 17.2 \\
AWM-iter & 46.6 & 15.5 & \underline{3.4} & 15.5 \\
\rowcolor{rowhl}
SynAct (1 iter) & \underline{58.0} & 13.2 & \textbf{6.3} & \underline{29.3} \\
\rowcolor{rowhl}
SynAct (best iter) & \textbf{67.2} (6) & \underline{19.5} (3) & \textbf{6.3} (1) & \textbf{32.8} (5) \\
\addlinespace[2pt]
\multicolumn{5}{@{}l}{\textit{InterCode-Bash ($n{=}40$)}}\\
base & 32.1 & 8.2 & 15.2 & 23.5 \\
ACE & 32.5 & 15.0 & 12.5 & \underline{22.5} \\
ReasoningBank & 42.5 & 17.5 & 15.0 & 20.0 \\
AWM & 27.5 & 10.0 & 12.5 & 20.0 \\
AWM-iter & 27.5 & 12.5 & 25.0 & \underline{22.5} \\
\rowcolor{rowhl}
SynAct (1 iter) & \underline{47.5} & \underline{20.0} & \underline{29.2} & \textbf{35.8} \\
\rowcolor{rowhl}
SynAct (best iter) & \textbf{48.3} (5) & \textbf{25.8} (5) & \textbf{40.0} (7) & \textbf{35.8} (1) \\
\addlinespace[2pt]
\multicolumn{5}{@{}l}{\textit{$\tau$-bench retail ($n{=}115$)}}\\
base & 27.4 & 25.7 & 5.4 & 13.1 \\
ACE & 24.3 & 27.0 & 2.6 & 19.1 \\
ReasoningBank & 26.1 & 20.9 & 6.1 & 19.1 \\
AWM & 31.3 & 24.3 & 5.4 & 13.0 \\
AWM-iter & 24.3 & 21.7 & 5.2 & 15.7 \\
\rowcolor{rowhl}
SynAct (1 iter) & \underline{37.7} & \underline{37.1} & \underline{7.5} & \underline{22.6} \\
\rowcolor{rowhl}
SynAct (best iter) & \textbf{42.6} (6) & \textbf{38.8} (7) & \textbf{8.7} (8) & \textbf{25.2} (6) \\
\bottomrule
\end{tabular}

\smallskip
\end{table}

This appendix provides the two full grids summarized in the body: the acquisition-function comparison of \S\ref{sec:res-signal} (Table~\ref{tab:signal}) and the per-iteration curve of Appendix~\ref{app:iter} (Table~\ref{tab:curve}).

\paragraph{What makes the acquisition signal useful?}
Table~\ref{tab:signal} holds everything but the selection rule fixed, isolating the ranking from the rest of the pipeline. Outcome-MI wins every column, and the alternatives are informative in different ways. Random selection reaches $+8.2$pp over base because the expert is useful wherever it is pointed: a randomly chosen decision still lands in a real failing trajectory, and the digest still describes something worth fixing. Action entropy reaches only $+6.8$pp despite deliberately preferring variable decisions, and Proposition~\ref{prop:sep} explains the gap: entropy cannot distinguish a cosmetic decision from an outcome-separating one when their arm counts agree. Part of the budget therefore buys blocks whose divergent arms reach the same end state, and the expert must then explain variation that the outcome does not resolve. Ranking on the wrong quantity can be worse than not ranking at all.

Success-MI reaches $+6.9$pp and exposes a different limitation. It fails outright on Gemma-3-4B TextCraft, where \emph{n/a} means no task contains both a winning and a losing rollout: under $\omega{=}\mathrm{succ}$ those nodes are outcome-constant, the score is identically zero by Clause 1, and the budget goes unspent. Outcome-MI retains information whenever different actions lead to distinguishable outcomes, including when every rollout fails.

\paragraph{What do successive iterations learn?}
Table~\ref{tab:curve} evaluates each iteration's rule independently rather than committing to a schedule in advance, which lets a row be read as a curve. Two regimes recur across the sixteen rows. When one convention captures the dominant failure mode, performance peaks early and plateaus, as on Qwen Bash, which gains $+15.4$pp at iteration~1 and holds within a point of that level. When training tasks contain several distinct modes, later iterations continue to add value: Qwen TextCraft climbs from $58.0$ at iteration~1 to $67.2$ at iteration~6. Only one row peaks at iteration~10: a fixed schedule would either stop before the peak or spend expert calls after it, which motivates evaluating each iteration's rule separately. The \emph{base} column moves across iterations because the gate re-measures the unaided agent each time, and we pool those measurements and treat differences of a few points as inseparable at 3 repeats.

\begin{table}[h]
\centering
\small
\caption{\textbf{Acquisition-function comparison} (accuracy \%, 3-repeat mean), referenced from \S\ref{sec:res-signal}. Everything but the selection rule is held fixed, using the same rollouts, the same authoring call, and the same gate. \emph{n/a} means the signal acquired too little usable evidence to author a rule at all, a failure mode unique to the success-based score, which requires a task to contain both a winning and a losing rollout. The shaded row is ours, and bold marks the best rule per column.}
\label{tab:signal}
\setlength{\tabcolsep}{5.5pt}
\begin{tabular}{@{}llrrrr@{}}
\toprule
Benchmark & Acquisition & Qwen3-32B & Ministral-3-3B & Gemma-3-4B & Gemma-3-12B \\
\midrule
\multirow{5}{*}{SQL}
 & base                       & 64.4 & 49.1 & 45.8 & 64.9 \\
 & SynAct\nobreakdash-Rand (random)         & 69.3 & 64.9 & 50.0 & 67.3 \\
 & SynAct\nobreakdash-Ent ($H(A)$)         & 69.3 & 62.2 & 51.3 & 65.6 \\
 & SynAct\nobreakdash-Succ ($I(A;\mathrm{succ})$) & 68.2 & 61.3 & 49.6 & 65.3 \\
 \rowcolor{rowhl}
 & SynAct ($I(A;O)$)       & \textbf{75.6} & \textbf{65.1} & \textbf{52.9} & \textbf{71.3} \\
\midrule
\multirow{5}{*}{TextCraft}
 & base                       & 41.2 & 13.0 & 1.7 & 17.5 \\
 & SynAct\nobreakdash-Rand (random)        & 58.8 & 20.9 & 3.3 & 28.8 \\
 & SynAct\nobreakdash-Ent ($H(A)$)         & 57.6 & 16.9 & 4.5 & 24.3 \\
 & SynAct\nobreakdash-Succ ($I(A;\mathrm{succ})$) & 50.8 & 20.9 & \emph{n/a} & 28.2 \\
 \rowcolor{rowhl}
 & SynAct ($I(A;O)$)       & \textbf{67.2} & \textbf{27.1} & \textbf{6.2} & \textbf{32.2} \\
\midrule
\end{tabular}
\end{table}

\begin{table}[h]
\centering
\caption{\textbf{SynAct test accuracy (\%) at each training iteration} (3-repeat mean), grouped by benchmark with one row per agent model, referenced from \S\ref{sec:res-budget} and Appendix~\ref{app:iter}. The \emph{base} column pools the unaided measurement over the completed iterations, because the gate re-measures the base each iteration.}
\label{tab:curve}
\footnotesize
\setlength{\tabcolsep}{4pt}
\begin{tabular}{l l r rrrrrrrrrr}
\toprule
Benchmark & Model & base & 1 & 2 & 3 & 4 & 5 & 6 & 7 & 8 & 9 & 10 \\
\midrule
\multirow{4}{*}{InterCode-SQL} & Qwen3-32B & 66.8 & 68.0 & 68.2 & 69.3 & 68.4 & 69.3 & 68.4 & 68.2 & 69.3 & 69.6 & 68.4 \\
 & Ministral-3-3B & 50.9 & 66.7 & 66.9 & 64.7 & 65.8 & 68.9 & 68.9 & 67.8 & 67.8 & 67.8 & 66.9 \\
 & Gemma-3-4B & 47.6 & 50.7 & 50.9 & 50.0 & 50.9 & 50.9 & 52.0 & 49.3 & 50.7 & 48.4 & 49.8 \\
 & Gemma-3-12B & 64.6 & 65.6 & 66.0 & 66.7 & 65.1 & 66.0 & 65.8 & 66.7 & 66.9 & 66.2 & 67.6 \\
\midrule
\multirow{4}{*}{TextCraft} & Qwen3-32B & 46.6 & 58.0 & 59.8 & 57.5 & 61.5 & 57.5 & 67.2 & 64.9 & 58.6 & 66.1 & 67.2 \\
 & Ministral-3-3B & 13.0 & 13.2 & 18.4 & 19.5 & 15.5 & 18.4 & 14.9 & 15.5 & 19.5 & 16.7 & 19.5 \\
 & Gemma-3-4B & 2.8 & 6.3 & 4.0 & 5.7 & 6.3 & 4.6 & 4.0 & 5.7 & 2.9 & 5.7 & 4.6 \\
 & Gemma-3-12B & 19.3 & 29.3 & 29.3 & 31.0 & 29.9 & 32.8 & 21.8 & 27.6 & 24.7 & 24.7 & 28.7 \\
\midrule
\multirow{4}{*}{InterCode-Bash} & Qwen3-32B & 32.1 & 47.5 & 43.3 & 44.2 & 44.2 & 48.3 & 47.5 & 46.7 & 47.5 & 47.5 & 47.5 \\
 & Ministral-3-3B & 8.2 & 20.0 & 17.5 & 19.2 & 19.2 & 25.8 & 20.0 & 18.3 & 16.7 & 20.8 & 13.3 \\
 & Gemma-3-4B & 15.2 & 29.2 & 32.5 & 31.7 & 28.3 & 28.3 & 39.2 & 40.0 & 31.7 & 32.5 & 34.2 \\
 & Gemma-3-12B & 23.5 & 35.8 & 33.3 & 23.3 & 28.3 & 28.3 & 25.0 & 26.7 & 28.3 & 25.0 & 25.0 \\
\midrule
\multirow{4}{*}{$\tau$-bench retail} & Qwen3-32B & 27.4 & 37.7 & 36.5 & 38.8 & 40.3 & 38.8 & 42.6 & 37.4 & 40.6 & 38.8 & 36.5 \\
 & Ministral-3-3B & 25.7 & 37.1 & 37.1 & 36.2 & 35.7 & 35.4 & 38.0 & 38.8 & 37.4 & 36.5 & 37.7 \\
 & Gemma-3-4B & 5.4 & 7.5 & 8.1 & 7.0 & 6.7 & 7.0 & 8.1 & 7.8 & 8.7 & 8.7 & 8.4 \\
 & Gemma-3-12B & 13.1 & 22.6 & 22.3 & 22.3 & 24.1 & 25.2 & 25.2 & 22.6 & 23.5 & 20.9 & 23.5 \\
\bottomrule
\end{tabular}
\end{table}

\section{Cost of a Routed Expert Budget}
\label{app:cost-detail}

\paragraph{Agent tokens against expert tokens.}
What does the budget look like when agent and expert computation are priced separately? SynAct spends \emph{more total} tokens, running the agent $k{=}8$ times per training task, and roughly $15\times$ \emph{fewer expert} tokens, $0.13\%$ of its own budget, calling the expert once per iteration rather than once or twice per training task. In turns the separation widens to one expert turn per iteration against 50 to 100 per run (Appendix~\ref{app:cost}). The two resources move in opposite directions, and total token count obscures the relevant trade.

\paragraph{Substituting a human expert.}
The asymmetry matters most when the expert is a person rather than a model (\S\ref{sec:setup}). Human throughput is bound by the interventions that can be read and authored rather than by token price. A method whose expert cost scales with the training set therefore becomes increasingly expensive to operate, whereas one intervention per iteration stays feasible under either instantiation. We measure the per-call load directly: the digest for the worked example of Appendix~\ref{app:example} spans $49$ blocks and roughly $27$k characters, a reading task of minutes rather than hours, and the returned convention is $2$--$4$ lines. Expert turns are consequently the more meaningful unit for human deployment, which is why we report them alongside tokens.

\section{Cost Accounting Details}
\label{app:cost}
\begin{table}[t]
\centering
\small
\setlength{\tabcolsep}{6pt}
\caption{\textbf{Training cost per (model $\times$ dataset) run} at $n_{\text{train}}{=}50$ and $k{=}8$. Expert tokens are the expensive ones, costing $10$--$30\times$ the agent's per-token rate. Token counts are estimates and turn counts are exact.}
\label{tab:cost}
\begin{tabular}{@{}lrrr@{}}
\toprule
& \multicolumn{2}{c}{Tokens} & Expert \\
\cmidrule(lr){2-3}
Setting & Agent & Expert & turns \\
\midrule
SynAct, 1 iter.   & 7.6M  & \textbf{0.01M} & \textbf{1} \\
SynAct, 3 iter.   & 22.9M & \textbf{0.03M} & 3 \\
SynAct, 10 iter.  & 76.2M & \textbf{0.10M} & 10 \\
\midrule
ACE (1 epoch)        & 0.95M & 0.20M & 100 \\
ReasoningBank        & 0.95M & 0.15M & 50 \\
\bottomrule
\end{tabular}
\end{table}

Table~\ref{tab:cost} separates agent and expert tokens because they price differently. At one iteration SynAct spends $0.01$M expert tokens against ReasoningBank's $0.15$M and ACE's $0.20$M, a $15$ to $20\times$ reduction. The ratio narrows to $5$ to $6.7\times$ at three iterations, since each additional SynAct iteration adds one expert call while the baselines pay their expert cost once. The largest savings therefore occur in the same regime where most of the accuracy gain is obtained, since iteration~1 captures most of the available improvement (Appendix~\ref{app:iter}), and the headline $15\times$ should consequently be read as a single-iteration figure rather than an iteration-independent property.

The token ratios also hide a second asymmetry. Expert tokens account for $0.13\%$ of SynAct's own budget at every iteration count, against $17\%$ for ACE and $14\%$ for ReasoningBank: a rise in the expert's price per token barely moves SynAct's total and moves the baselines' substantially. Expert turns diverge even more sharply, one call per iteration against 50 to 100 per baseline run, and for a human expert the turn count is the binding resource. This is why we report both units.

The remaining cost is agent-side. SynAct spends $8\times$ more agent tokens at one iteration because it runs $k{=}8$ rollouts per training task to obtain the acquisition scores. This cost is reducible: Appendix~\ref{app:budget} shows $k$ can be cut to $2$ while still beating base on all six settings, whereas the baselines' expert cost is not reducible without changing the methods. Turn counts are exact, taken from cached traces, and token counts are estimates formed from exact call counts and per-call anchors measured live.

\section{Inference-Time Routing Details}
\label{app:online}

\paragraph{Where should the expert intervene?}
At inference time the question shifts: instead of deciding which training decisions deserve explanation, we decide which test tasks deserve intervention. For each task we draw $k{=}8$ rollouts at temperature $0.7$ and group them by state key as in \S\ref{sec:acq}. A consistency gate decides whether intervention is warranted: we extract an outcome key per rollout, take the modal key, and route only when the modal fraction falls below $\theta{=}0.6$. A task that agrees with itself is answered by majority vote at no cost, and a task already solved by any pooled rollout is kept as is. The expert thus sees only tasks where the agent both disagrees with itself and has not already succeeded. On a routed task, $\mathcal{I}(v^\star)$ selects the decision node, the expert returns a correction anchored at that node, and the agent retries with the correction as the first of $R$ attempts. Attempts $2,\dots,R$ are ordinary self-reflection.

\paragraph{Can other routers spend the same budget as effectively?}
We cap every expert-using alternative in Table~\ref{tab:online} to SynAct's expert-call count $B$ at this operating point, making the comparison like for like. The cascade escalates whole tasks to the expert under the same expert-token spend. The action-entropy router keeps the consistency gate but replaces $\mathcal{I}(v^\star)$ with $\arg\max_v \delta H(v)$, spending its $B$ expert plans on the most uncertain tasks. The judge hands the routing decision to the expert itself, capped at $B$ judge calls and averaged over task-arrival orders, since the judge cannot rank tasks it has not yet seen. SynAct-rules additionally activates the iteration-10 convention from \S\ref{sec:res-budget} during rollouts and retries.

\begin{table}[h]
\centering
\caption{\textbf{Inference-time routing at a matched expert budget} ($\theta{=}0.6$, $R{=}3$), held-out accuracy (\%), referenced from \S\ref{sec:res-online}. SR(10) is self-reflection with ten retries and no expert. Cascade, action-entropy, and the expert judge are each throttled to SynAct's expert budget at this operating point. Bold marks the best routing method per row.}
\label{tab:online}
\footnotesize
\setlength{\tabcolsep}{4.5pt}
\begin{tabular}{l r r r r r r r}
\toprule
Model & base & SR(10) & Cascade & Act.-Ent. & Judge & SynAct & +rules \\
\midrule
\multicolumn{8}{@{}l}{\textit{InterCode-SQL}}\\
Qwen3-32B & 66.0 & 82.7 & 70.7 & 86.0 & 84.7 & 85.3 & \textbf{91.3} \\
Ministral-3-3B & 47.3 & 82.0 & 61.3 & \textbf{88.7} & 85.9 & 87.3 & 88.0 \\
Gemma-3-4B & 48.7 & 62.0 & 49.3 & 67.3 & 66.3 & \textbf{68.0} & 66.0 \\
Gemma-3-12B & 64.7 & 83.3 & 65.3 & 86.0 & 84.3 & 85.3 & \textbf{86.7} \\
\addlinespace[2pt]
\multicolumn{8}{@{}l}{\textit{TextCraft}}\\
Qwen3-32B & 50.0 & 72.4 & 55.2 & 75.9 & 80.0 & 81.0 & \textbf{86.2} \\
Ministral-3-3B & 19.0 & 31.0 & 10.3 & 31.0 & 29.4 & 29.3 & \textbf{41.4} \\
Gemma-3-4B & 3.4 & 5.2 & 1.7 & 17.2 & 10.3 & 12.1 & \textbf{19.0} \\
Gemma-3-12B & 15.5 & 32.8 & 25.9 & 51.7 & 45.9 & 50.0 & \textbf{62.1} \\
\addlinespace[2pt]
\multicolumn{8}{@{}l}{\textit{InterCode-Bash}}\\
Qwen3-32B & 37.5 & 60.0 & 40.0 & 65.0 & 64.4 & 67.5 & \textbf{75.0} \\
Ministral-3-3B & 10.0 & 35.0 & 12.5 & 45.0 & 46.1 & 50.0 & \textbf{62.5} \\
Gemma-3-4B & 10.0 & 45.0 & 22.5 & \textbf{55.0} & 51.9 & 52.5 & \textbf{55.0} \\
Gemma-3-12B & 27.5 & 47.5 & 32.5 & 52.5 & 55.1 & 52.5 & \textbf{57.5} \\
\addlinespace[2pt]
\multicolumn{8}{@{}l}{\textit{$\tau$-bench retail}}\\
Qwen3-32B & 25.2 & 65.2 & 49.6 & 67.8 & 67.9 & \textbf{69.6} & 63.5 \\
Ministral-3-3B & 27.8 & 62.6 & 43.5 & 58.3 & 61.9 & 60.9 & \textbf{64.3} \\
Gemma-3-4B & 5.2 & 9.6 & 8.7 & 13.0 & 11.3 & 11.3 & \textbf{17.4} \\
Gemma-3-12B & 12.2 & 29.6 & 20.9 & 33.9 & 37.6 & 38.3 & \textbf{47.0} \\
\bottomrule
\end{tabular}
\end{table}

\begin{figure}[t]
\centering
\includegraphics[width=\linewidth]{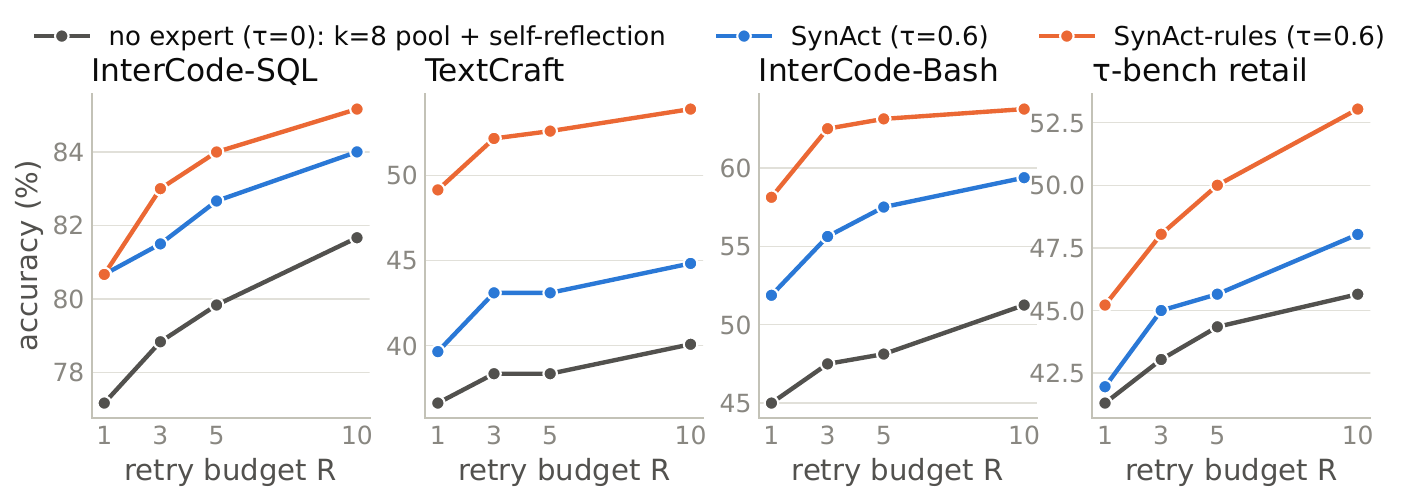}
\caption{\textbf{Accuracy against retry budget $R$}, mean over the four agent models. The no-expert pool with self-reflection against SynAct and SynAct-rules at $\theta{=}0.6$.}
\label{fig:online-retry}
\end{figure}

\paragraph{How sensitive is routing to the consistency threshold?}
Sweeping $\theta$ from $0.5$ to $0.8$ at $R{=}3$ expands the routed pool at every step, since $k{=}8$ restricts the modal fraction to multiples of $1/8$. Across the two SQL settings, accuracy changes by at most $1.3$pp while expert calls roughly triple, leaving little for the threshold to decide once the pool is near saturation. TextCraft with Qwen is the one swept setting where the knob matters: accuracy rises $+6.9$pp from $\theta{=}0.5$ to $0.7$ before the routable pool saturates. TextCraft with Ministral is flat because every threshold selects the same capability-bound pool. We therefore use $\theta{=}0.6$ for every benchmark and model without per-setting tuning: where routing helps, a broad range of thresholds captures the gain.

\paragraph{What if the agent gets more chances after intervention?}
Table~\ref{tab:online-r} raises the retry budget from $R{=}3$ to $R{=}10$ at the same expert budget and operating point, matching the retry allowance of SR(10). Accuracy never falls and rises in 11 of 16 settings, with the largest gain $+11.3$pp for Ministral on $\tau$-bench. The correction and self-reflection therefore compose rather than compete: the expert supplies a targeted change, and additional attempts continue the search from there. Figure~\ref{fig:online-retry} makes the comparison across the full retry range, and the routed variants dominate the no-expert pool at matched $R$ throughout: the benefit of expert feedback does not vanish once the agent is allowed more retries.

\begin{table}[h]
\centering
\caption{\textbf{SynAct at $R{=}10$ against $R{=}3$}, same expert budget and operating point ($\theta{=}0.6$), accuracy (\%).}
\label{tab:online-r}
\small
\setlength{\tabcolsep}{5pt}
\begin{tabular}{l rr @{\hspace{14pt}} l rr}
\toprule
SQL & $R{=}3$ & $R{=}10$ & TextCraft & $R{=}3$ & $R{=}10$ \\
\midrule
Qwen3-32B & 85.3 & 86.0 & Qwen3-32B & 81.0 & 82.8 \\
Ministral-3-3B & 87.3 & 87.3 & Ministral-3-3B & 29.3 & 31.0 \\
Gemma-3-4B & 68.0 & 74.0 & Gemma-3-4B & 12.1 & 12.1 \\
Gemma-3-12B & 85.3 & 88.7 & Gemma-3-12B & 50.0 & 53.4 \\
\midrule
Bash & $R{=}3$ & $R{=}10$ & $\tau$-bench & $R{=}3$ & $R{=}10$ \\
\midrule
Qwen3-32B & 67.5 & 67.5 & Qwen3-32B & 69.6 & 70.4 \\
Ministral-3-3B & 50.0 & 57.5 & Ministral-3-3B & 60.9 & 72.2 \\
Gemma-3-4B & 52.5 & 57.5 & Gemma-3-4B & 11.3 & 11.3 \\
Gemma-3-12B & 52.5 & 55.0 & Gemma-3-12B & 38.3 & 38.3 \\
\bottomrule
\end{tabular}
\end{table}

\section{Routing Signal at Inference: Outcome-MI Against Action Entropy}
\label{app:signal-routing}

The training loop is robust to the acquisition ranking (Appendix~\ref{app:budget}). We therefore ask where the choice of signal does matter. Inference-time routing is the natural place to look: the signal is spent per task rather than amortized into a reusable rule, and a misplaced branch point cannot be recovered through aggregation across tasks. We compare SynAct's outcome-MI branch point with the action-entropy alternative ($\arg\max_v \delta H(v)$) on $\tau$-bench retail from two views. First, we run both routers on the full benchmark at the shared operating point ($\theta{=}0.6$, $R{=}3$). Full-benchmark accuracy mixes the routing decision with the consistency gate and the unrouted pool. The second view therefore isolates the signal itself: both routers admit exactly the same $B$ tasks through the consistency gate, differing only in \emph{where} in the trajectory the expert plan is requested. We report the resulting rescue rate, solved over routed, on this shared task set. Gemma-3-4B is excluded because its near-floor accuracy leaves too few rescues for a meaningful comparison.

\begin{table}[h]
\centering
\caption{\textbf{Outcome-MI against action entropy as the routing signal} ($\tau$-bench retail, $\theta{=}0.6$, $R{=}3$, $n{=}115$). Left: full-benchmark accuracy (\%). Right: expert rescues among the $B$ tasks both signals route, solved/routed (\%). Both signals route the same $B$ tasks per model, and the rescue columns isolate \emph{where} the expert plan is anchored. Bold marks the better signal per view.}
\label{tab:signal-routing}
\small
\setlength{\tabcolsep}{5pt}
\begin{tabular}{l rr @{\hspace{16pt}} r rr}
\toprule
 & \multicolumn{2}{c}{Full benchmark (acc.\ \%)} & & \multicolumn{2}{c}{Rescues on routed set} \\
\cmidrule(lr){2-3}\cmidrule(l){5-6}
Model & MI-out & Act.-Ent. & $B$ & MI-out & Act.-Ent. \\
\midrule
Qwen3-32B      & \textbf{69.6} & 67.8 & 33 & \textbf{7/33} (21.2) & 5/33 (15.2) \\
Ministral-3-3B & \textbf{60.9} & 58.3 & 44 & \textbf{9/44} (20.5) & 6/44 (13.6) \\
Gemma-3-12B    & \textbf{38.3} & 33.9 & 61 & \textbf{8/61} (13.1) & 3/61 (4.9) \\
\bottomrule
\end{tabular}
\end{table}


On the full benchmark, outcome-MI leads action entropy by $1.8$--$4.4$pp across the three models, although this gap reflects both the routing signal and the rest of the routing pipeline. The controlled comparison is more revealing: on the shared routed set, where the only difference is the branch presented to the expert, outcome-MI rescues $13$--$21\%$ of routed tasks, compared with $5$--$15\%$ for action entropy, a $1.4\times$--$2.7\times$ advantage that widens on the weakest model. This is the regime captured by Proposition~\ref{prop:sep}: entropy identifies the step where the agent is most uncertain, which need not be the step that determines the outcome. Anchoring the expert's single correction at a cosmetic branch therefore spends the consultation without changing where the run ends.

The contrast with training-time robustness is instructive. Aggregating evidence into a digest lets the expert recover from imperfect selection, whereas a per-task consultation has no such slack. The value of a consequence-aware signal therefore concentrates where feedback cannot be amortized: at inference, choosing \emph{where} to intervene is itself the bottleneck.

\end{document}